%% file: main.tex
\pdfoutput=1
\documentclass[11pt]{article}
\usepackage[margin=1.1in]{geometry}
\usepackage{amsmath,amssymb,amsthm,mathtools}
\usepackage{xcolor}
\usepackage{microtype}
\usepackage[round]{natbib}
\usepackage{tikz,float}
\usepackage{hyperref}
\usepackage{url}
\hypersetup{
  colorlinks=true,linkcolor=blue!60!black,citecolor=blue!60!black,urlcolor=blue!60!black,
  pdftitle={Exact Recovery Thresholds for Weighted Data Selection in Vector-Valued Linear Regression},
  pdfauthor={Guangjian Zhang}
}

\theoremstyle{plain}
\newtheorem{theorem}{Theorem}[section]
\newtheorem{lemma}[theorem]{Lemma}
\newtheorem{proposition}[theorem]{Proposition}
\newtheorem{corollary}[theorem]{Corollary}
\newtheorem{conjecture}[theorem]{Conjecture}
\theoremstyle{definition}
\newtheorem{definition}[theorem]{Definition}
\newtheorem{remark}[theorem]{Remark}

\newenvironment{restate}[1]{\par\medskip\noindent\textbf{#1.}\ \itshape}{\par\medskip}

\newcommand{\R}{\mathbb{R}}
\newcommand{\C}{\mathbb{C}}
\newcommand{\E}{\mathbb{E}}
\newcommand{\Fw}{F_{\mathrm{weighted}}}
\newcommand{\Fsc}{F^{\mathrm{sc}}_{\mathrm{weighted}}}
\newcommand{\ones}{\mathbf{1}}

\newcommand{\supp}{\operatorname{supp}}
\newcommand{\conv}{\operatorname{conv}}
\newcommand{\cone}{\operatorname{cone}}
\newcommand{\spn}{\operatorname{span}}
\newcommand{\rank}{\operatorname{rank}}
\newcommand{\dist}{\operatorname{dist}}
\newcommand{\tr}{\operatorname{tr}}
\newcommand{\relint}{\operatorname{relint}}
\newcommand{\Diag}{\operatorname{Diag}}
\newcommand{\nstar}{n^{\star}}
\newcommand{\Lstar}{L^{\star}_{D}}
\newcommand{\Wnorm}[1]{\lVert #1\rVert}

\title{Exact Recovery Thresholds for Weighted Data Selection\\ in Vector-Valued Linear Regression\thanks{The
scalar case $m=1$ of the problem is treated in our companion paper \citep{ZhangScalar},
arXiv:2608.28007, by the same author.}}

\author{Guangjian Zhang\\ University of New South Wales\\ \texttt{zgj1226029469@outlook.com}}
\date{October 2, 2026}

\begin{document}

\maketitle

\begin{center}\small
\textbf{arXiv.org perpetual, non-exclusive license}\\[3pt]
\begin{minipage}{0.88\textwidth}\centering
Copyright \textcopyright\ 2026 Guangjian Zhang. Distributed by arXiv.org under
the arXiv.org perpetual, non-exclusive license,
\url{https://arxiv.org/licenses/nonexclusive-distrib/1.0/}.
\end{minipage}\\[5pt]
arXiv:2608.30254\quad DOI: \url{https://doi.org/10.48550/arXiv.2608.30254}
\end{center}

\begin{abstract}
We resolve the threshold part of Question~4 of the COLT~2025 open problem
``Data Selection for Regression Tasks'' of Hanneke, Moran, Shlimovich and
Yehudayoff. We study vector-valued linear regression with square loss
$\ell_{(x,y)}(W)=\Wnorm{Wx-y}_2^2$, where $x\in\R^d$ and $y\in\R^m$.
The learner returns the minimum-Frobenius-norm empirical risk minimizer.
We prove that the minimal budget of weighted examples for recovering the full-data loss
on \emph{every} finite dataset is exactly $\nstar(d,m)=(m+1)d$. We
determine the weighted selection profile $\Fw(d,m,n)$ at the near-threshold
budget: $\Fw(d,m,(m+1)d-1)=1+1/(dm^2)$. We recover the known spanning-budget
value $\Fw(d,m,d)=d+1$ for every $m$, and $\Fw(d,m,n)=\infty$
for $n<d$. For the smallest open intermediate cell $(d,m)=(2,2)$ we prove
$\Fw(2,2,3)\in[13/8,\,15/8]$ and $\Fw(2,2,4)\in[5/4,\,3/2]$. We reduce
the conjectured exact values $13/8$ and $5/4$ to a finite moment problem on
the circle with at most seven atoms and assemble structural evidence for it.
The upper bounds use a fixed-basis conic compression lemma, a
determinant--facet rigidity theorem for maximal certificates, and sharp
sparsification lemmas for zero-mean weighted point systems. These tools may
be of independent interest. We also exhibit an explicit six-point integer
dataset with $d=m=2$ on which \emph{no} weighted selection of $2d$ points
recovers the optimal loss. Thus the scalar sufficient budget $2d$ does not
extend to vector-valued outputs. Our new regression-profile results for $m\ge2$ extend the scalar theory
for $m=1$.
\end{abstract}

\input{sec-intro}
\input{sec-setup}
\input{sec-threshold}

\input{sec-rigidity}

\input{sec-near}
\input{sec-remaining}

\input{sec-discussion}

\subsection*{Proofs and exact verification}
All results are mathematical. Complete proofs of every theorem, lemma and
proposition of Sections~\ref{sec:prelim}--\ref{sec:intermediate} are given
in Appendices~\ref{app:prelim}--\ref{sec:evidence}. The general
sparsification theorem is proved in Appendix~\ref{sec:kpoint}.
Conjecture~\ref{conj:CE} remains open. Appendix~\ref{sec:evidence}
provides partial results and evidence for it. The explicit instances in
Definition~\ref{def:axial} and Proposition~\ref{prop:lower22}, along with
the moment systems of Appendix~\ref{sec:evidence}, are given in closed form
with rational or algebraic coordinates. Every numerical value in the paper
can therefore be recomputed by exact arithmetic from the formulas in the
text. These include $43/40$, $23/20$, $13/8$, $5/4$, $1/2$, $4/7$, $2/7$, $15/23$, the
hexagon four-point value $0$, and the constants of
Lemma~\ref{lem:mmean} and Theorem~\ref{thm:kpoint}. We performed these exact
computations as a supplementary check. No numerical search is used in the proof of any theorem.

\bibliography{refs}
\bibliographystyle{plainnat}

\appendix
\input{app-prelim}
\input{app-threshold}
\input{app-lowbudget}
\input{app-near}

\input{app-cell}
\input{app-evidence}
\input{app-kpoint}

\input{expression-proof-locations}
\end{document}

%% file: sec-intro.tex
\section{Introduction}\label{sec:intro}

Data curation and reusable representations \citep{DataComp2023,DINOv2}
call for precise theoretical benchmarks for data selection.

How well can a fixed, natural learning rule learn from
only $n$ examples selected from a larger dataset? This question
underlies coreset construction, data pruning and subsampling for regression.
A typical multi-output instance is a linear head trained on a fixed
representation, where all outputs must share one selected set of examples
and one weight vector. \citet{HMSYnote} posed a worst-case version for basic regression
in a COLT~2025 open-problem note, with a companion paper
\citep{HMSYlong}. For scalar linear regression they determined the endpoint
regimes under weighted selection and the minimum-norm empirical risk
minimizer. The worst-case ratio between the selected model's loss and the
optimal full-data loss is infinite for $n<d$, equals $d+1$ at $n=d$, and
equals $1$ for $n\ge2d$ \citep[Theorem~1]{HMSYnote}. The intermediate range
$d<n<2d$ was posed as their Question~2. Their note concludes with a unified
\emph{vector-valued} formulation using predictors $W\colon\R^d\to\R^m$,
square loss and the minimum \emph{Frobenius}-norm ERM. It asks, as Question~4:
\begin{quote}\itshape
Given $d,m,n$, what is the value of $\Fw(d,m,n)$? In particular, what is the
minimal $n=n(d,m)$ such that this ratio equals $1$?
\end{quote}
This paper resolves the threshold question completely and determines several
further values of the profile $\Fw(d,m,\cdot)$. Throughout, a dataset
$D=(z_i)_{i=1}^N$ is a finite sequence of points
$z_i=(x_i,y_i)\in\R^d\times\R^m$. We write $\ell_z(W)=\Wnorm{Wx-y}_2^2$,
$L_D(W)=\frac1N\sum_i\ell_{z_i}(W)$ and $\Lstar=\min_WL_D(W)$. The learner $A$ returns
the empirical risk minimizer of minimal Frobenius norm. Weighted selection of
budget $n$ chooses $n$ points of $D$ (repetitions allowed) and a convex
combination $F$ of their losses. Its loss on the full dataset is $L_D(A(F))$.
The profile $\Fw(d,m,n)$ is the supremum over datasets of
$\inf_FL_D(A(F))/\Lstar$, with degenerate-ratio conventions fixed in
Section~\ref{sec:prelim}. The exact-recovery threshold is
$\nstar(d,m)=\min\{n:\Fw(d,m,n)=1\}$.

\paragraph{Contributions.} All statements below are for $d,m\ge2$. They hold
verbatim for $d,m\ge1$ (Remark~\ref{rem:range}).
\begin{list}{}{\setlength{\leftmargin}{0pt}\setlength{\labelwidth}{0pt}\setlength{\labelsep}{0pt}\setlength{\itemindent}{0pt}\setlength{\listparindent}{0pt}\setlength{\itemsep}{1pt}\setlength{\parsep}{0pt}\setlength{\topsep}{3pt}}
\item \textbf{Exact threshold} (Theorem~\ref{thm:threshold}):
  $\nstar(d,m)=(m+1)d$. The upper bound is data-dependent: every dataset of
  feature rank $r\ge1$ is recovered exactly by $(m+1)r$ weighted points
  (Theorem~\ref{thm:upper}). The lower bound is an explicit integer dataset on
  $(m+1)d$ points whose optimal ratio at budget $(m+1)d-1$ is exactly
  $1+\frac{m+1}{2md(m^2+m-1)}$ (Theorem~\ref{thm:axial}).
\item \textbf{Near-threshold value} (Theorem~\ref{thm:nearthreshold}):
  $\Fw(d,m,(m+1)d-1)=1+1/(dm^2)$. The proof hinges on rigidity of
  maximal exact certificates (Theorem~\ref{thm:rigidity}). Every
  dataset unrecoverable within budget $(m+1)d-1$ has a minimal
  spanning zero certificate whose support decomposes into $d$ disjoint
  $(m+1)$-point residual circuits on independent feature lines. All
  nonzero features of such a dataset lie on these lines
  (Theorem~\ref{thm:lines}). A sharp $m$-point mean lemma
  (Lemma~\ref{lem:mmean}) then finishes the argument.
\item \textbf{Low budgets} (Theorem~\ref{thm:lowbudget}): $\Fw(d,m,n)=\infty$
  for $n<d$ and $\Fw(d,m,d)=d+1$, \emph{independently of $m$}, as a direct
  consequence of \citet{DW17} and the scalar lower bounds. A shared
  volume-sampling subset serves all output rows.
\item \textbf{The smallest open cell} $(d,m)=(2,2)$ (Section~\ref{sec:intermediate}):
  $\Fw(2,2,3)\in[13/8,15/8]$ and $\Fw(2,2,4)\in[5/4,3/2]$
  (Theorem~\ref{thm:intervals}). The conjectured values $13/8$ and $5/4$ are
  reduced to a finite-dimensional moment problem with at most seven atoms on
  the unit circle (Proposition~\ref{prop:seven}). For this moment problem,
  we prove a sharp result for the two-direction class, stratified local maximality, two
  exact obstructions, and closure of several classes
  (Appendix~\ref{sec:evidence}).
\item \textbf{A byproduct}: the scalar sufficient budget $2d$ fails for every
  $m\ge2$. On an explicit six-point integer dataset with $d=m=2$, every
  weighted selection of four points has loss at least $\frac{23}{20}\Lstar$
  (Proposition~\ref{prop:2dfails}).
\end{list}
Table~\ref{tab:profile} summarizes the resulting picture of the profile.

\begin{table}[!hbp]
\caption{The weighted selection profile $\Fw(d,m,n)$ for the
minimum-Frobenius-norm ERM ($d,m\ge2$; the scalar specializations $m=1$ are
due to \citet{HMSYnote}, \citet{OneShort} and \citet{ZhangScalar}).}
\label{tab:profile}
\begin{center}
\begin{tabular}{ll}
\multicolumn{1}{c}{\bf budget $n$} & \multicolumn{1}{c}{\bf $\Fw(d,m,n)$}\\ \hline \\[-6pt]
$n<d$ & $\infty$ (Theorem~\ref{thm:lowbudget})\\
$n=d$ & $d+1$ (Theorem~\ref{thm:lowbudget})\\
$d<n<(m+1)d-1$ & open; at $(2,2)$: $[13/8,15/8]$, $[5/4,3/2]$ for $n=3,4$ (Theorem~\ref{thm:intervals})\\
$n=(m+1)d-1$ & $1+1/(dm^2)$ (Theorem~\ref{thm:nearthreshold})\\
$n\ge(m+1)d$ & $1$ (Theorem~\ref{thm:threshold})
\end{tabular}
\end{center}
\end{table}

\paragraph{Techniques.} The threshold upper bound uses a
\emph{fixed-basis conic compression} lemma (Lemma~\ref{lem:compression}).
After we pin a feature basis of $r$ points, the negated sum of their
residual dyads lies in the cone of the remaining dyads. Conic
Carath\'eodory represents it with at most $mr$ of them, giving an exact
certificate on $(m+1)r$ points. The argument is shorter than the
Steinitz-based scalar route of \citet{HMSYlong} and is data-dependent. The
near-threshold upper bound uses a \emph{determinant--facet rigidity
theorem} (Theorem~\ref{thm:rigidity}). Consider a certificate of maximal size
$(m+1)d$. Its weighted feature Gram determinant is a polynomial
of degree at most $d$ in the weights, although the dyads span an
$md$-dimensional space. It vanishes on every facet of the normalized
certificate polytope. This forces a simplex and decomposes the
certificate into $d$ disjoint $(m+1)$-point circuits on $d$ independent
feature lines. A sharp \emph{$m$-point mean lemma} for zero-mean weighted
point systems then finishes the assembly. For the $(2,2)$ cell, a complex-variable
dictionary turns whitened instances into moment systems on the unit circle
and the selection cost into a convex combination of \emph{two-point
interpolation coefficients} (Lemma~\ref{lem:interp}).

\subsection{Related and concurrent work}\label{subsec:related}

\emph{Prior work.} The problem and the scalar endpoint results are due to
\citet{HMSYnote,HMSYlong}. The scalar value at $n=2d-1$ is proved in
\citet{OneShort} and \citet{ZhangScalar}, and further scalar intermediate
values in \citet{ZhangScalar}. The budget-$d$ upper
bound uses volume sampling and the exact expectation formulas of
\citet{DW17} (Theorem~5, used as an inequality; Section~\ref{sec:lowbudget}).
The general $k$-point mean lemma
(Theorem~\ref{thm:kpoint}) is proved in Appendix~\ref{sec:kpoint} by an
elementary argument. For $k=2$ and $M\ge3$, its constant coincides with, and
follows from, Theorem~2 of \citet{SharpPair}. That theorem covers neither $M=2$ nor
general $k$. \citet{Tinarrage} proves a radius-normalized, equal-weight
analogue. Carath\'eodory-type
compressions also underlie exact weighted coresets for least squares. For
example, the summaries of \citet{MJF19} preserve the covariance sums of the input.
An exact coreset preserves the objective for every parameter. Here the
summary must reproduce the model returned by a fixed learner retrained
on the selected original points. We obtain matching lower bounds on the
support and exact worst-case ratios below it.

\emph{Concurrent unrefereed preprints.} Recent unrefereed preprints treat
other questions of \citet{HMSYnote}: \citet{SharpPair} determines the
budget-$2$ column of Question~1, and \citet{OneShort} proves the scalar value
at budget $2d-1$ (Question~2). Neither manuscript treats the weighted
vector-valued Question~4. The relation of our
sparsification results to Theorem~2 of \citet{SharpPair} is stated above. We
make no priority claim on them.

\paragraph{Companion paper.} The scalar case $m=1$, Question~2 of
\citet{HMSYnote}, is the subject of our companion paper \citep{ZhangScalar},
available as arXiv:2608.28007. The two papers are by the same author.

\paragraph{Organization.} Sections~\ref{sec:prelim}--\ref{sec:remaining}
give results and proof sketches. Appendices~\ref{app:prelim}--\ref{sec:evidence}
give full proofs. Appendix~\ref{sec:kpoint} gives the general $k$-point mean lemma.

%% file: sec-setup.tex
\section{Recovery certificates and proof architecture}\label{sec:prelim}

Throughout, $d,m,n\ge1$ are integers. A dataset is a finite sequence
$D=(z_i)_{i=1}^N$ of points $z_i=(x_i,y_i)\in\R^d\times\R^m$, $N\ge1$,
with repetitions allowed. A linear predictor $W\in\R^{m\times d}$ has loss $\ell_z(W)=\Wnorm{Wx-y}_2^2$ and average loss
$L_D(W)=\frac1N\sum_{i=1}^N\ell_{z_i}(W)$. Write $\Lstar=\min_WL_D(W)$. The
learning rule $A$ maps a nonnegatively weighted objective
$F=\sum_ic_i\ell_{z_i}$ (with $c_i\ge0$, $\sum_ic_i=1$) to the minimum-Frobenius-norm minimizer of
$F$. For $c\in\R^N_{\ge0}$ write
$F_c=\sum_ic_i\ell_{z_i}$. Certificates below are nonzero and
unnormalized. For such $c$, we set $A(F_c):=A(F_{c/\sum_ic_i})$, the
minimal-norm minimizer of $F_c$ (positive rescaling preserves
the minimizer set). The smallest full-data loss achievable by
weighted selection with budget $n$ is
\[
  L^{\star}_D(n;\mathrm{weighted})
  =\inf_{\substack{z_{j_1},\dots,z_{j_n}\in D\\
      F\in\conv(\ell_{z_{j_1}},\dots,\ell_{z_{j_n}})}}L_D\bigl(A(F)\bigr),
\]
\[
  R_D(n)=\frac{L^{\star}_D(n;\mathrm{weighted})}{\Lstar},
  \qquad
  \Fw(d,m,n)=\sup_D\,R_D(n),
\]
where the supremum is over all datasets (all $N\ge1$), and
$\nstar(d,m)=\min\{n:\Fw(d,m,n)=1\}$. Following the convention of
\citet{HMSYlong}, $R_D(n)$ is $1$ when numerator and denominator are
both $0$, and $\infty$ when only the denominator is $0$. Since
$L_D(A(F))\ge\Lstar$ for every objective $F$, $R_D(n)\ge1$. We write
the inner optimization as an infimum (as in \citet{HMSYlong}): all
exact-recovery certificates below are attained, while some bounds in
Section~\ref{sec:intermediate} are limits of admissible selections. The
following normalization is immediate and used implicitly.

\begin{lemma}[Selection semantics]\label{lem:support}
Selections of budget $n$ are in value-preserving correspondence with weight
vectors $c\in\R^N_{\ge0}$, $\sum_ic_i=1$, $|\supp c|\le n$. Repeated picks
merge weights, and zero-weight picks can be discarded. Any $c$ with
$|\supp c|\le n$ is realizable with exactly $n$ slots by repeating a chosen
point and splitting its weight. In particular the feasible family grows
with $n$, so $L^{\star}_D(n;\mathrm{weighted})$, $R_D(n)$ and
$\Fw(d,m,n)$ are nonincreasing in $n$.
\end{lemma}

\paragraph{Geometry of the full-data minimizer.} Let
$T=\spn\{x_1,\dots,x_N\}$, $r=\dim T$ (the feature rank), and let $P_T$ be
the orthogonal projection onto $T$. Let $W^\circ$ be the
minimal-Frobenius-norm full-data minimizer, and set $\rho_i=W^\circ x_i-y_i$
(residuals) and $M_i=\rho_ix_i^\top\in\R^{m\times d}$ (residual dyads). The
objective and the squared Frobenius norm are additive across output rows,
so $A(F_c)$ consists of the minimal-$\ell_2$-norm scalar solutions row by
row, with the single shared weight vector $c$
(Lemma~\ref{lem:rows}). Moreover, $W^\circ=W^\circ P_T$, the first-order
condition $\sum_iM_i=0$ holds, and
\begin{equation}\label{eq:lossexp-main}
  L_D(W^\circ+H)=\Lstar+\tfrac1N\textstyle\sum_{i=1}^N\Wnorm{Hx_i}^2
  \qquad\text{for every }H\in\R^{m\times d}
\end{equation}
(Lemma~\ref{lem:geometry}). The two conditions in the next lemma have
separate roles: the zero-dyad condition makes $W^\circ$ stationary for the
selected objective, while the spanning condition makes the selected and
full-data minimizer sets coincide.

\input{stmt-lem-certificate}

We call a nonzero weight vector satisfying \eqref{eq:cert} a
\emph{spanning zero certificate}, and define the minimal certificate size
\begin{equation}\label{eq:tau}
  \tau(D)\;=\;\min\bigl\{|\supp c|:\ c\ \text{is a spanning zero
  certificate}\bigr\}.
\end{equation}
Certificates are built from positive circuits of the residual dyads.

\input{stmt-def-circuit}

Finally, scalar instances embed: if all $m$ coordinates of every label agree,
then every row of the vector problem is the same scalar problem, so
$\Fw(d,m,n)\ge\Fsc(d,n)$, the scalar weighted profile of \citet{HMSYnote}
(Lemma~\ref{lem:embed}).

\begin{remark}[Range of validity]\label{rem:range}
Nothing in the proofs of Theorems~\ref{thm:threshold},
\ref{thm:nearthreshold} and \ref{thm:lowbudget} uses $d\ge2$ or $m\ge2$:
they hold for all $d,m\ge1$. At $m=1$ they specialize to
$\nstar(d,1)=2d$ and $\Fw(d,1,2d-1)=1+\frac1d$, matching
\citet[Theorem~1]{HMSYnote}, \citet{OneShort} and \citet{ZhangScalar}. At
$d=1$ the whole profile is Corollary~\ref{cor:d1}. We state the theorems
for $d,m\ge2$ only to delineate the new contribution: for $m=1$ both
conclusions of Theorem~\ref{thm:lowbudget} and the threshold $2d$ are due to
\citet{HMSYnote,HMSYlong}, and the value at $n=2d-1$, stated in
\citet{HMSYnote}, is proved in \citet{OneShort} and \citet{ZhangScalar}.
\textbf{Our new regression-profile results concern the vector case $m\ge2$.}
\end{remark}
\paragraph{Proof routes.}
Selected reading links: \emph{Threshold:} Lemmas~\ref{lem:geometry}, \ref{lem:certificate},
\ref{lem:compression}; Theorems~\ref{thm:upper}, \ref{thm:axial},
\ref{thm:threshold} (Appendices~\ref{app:prelim}--\ref{app:threshold}). \emph{Near threshold:} Definition~\ref{def:circuit};
Theorems~\ref{thm:rigidity}, \ref{thm:lines}; Lemma~\ref{lem:mmean};
Theorem~\ref{thm:nearthreshold} (Appendix~\ref{app:near}).

%% file: stmt-lem-certificate.tex
\begin{lemma}[Exact certificate]\label{lem:certificate}
Let $c\in\R^N_{\ge0}$ be nonzero, with support $S$ (so $c_i>0$ on $S$). If
\begin{equation}\label{eq:cert}
  \sum_{i\in S}c_iM_i=0
  \qquad\text{and}\qquad
  \spn\{x_i:i\in S\}=T ,
\end{equation}
then $A(F_c)=W^\circ$. In particular, the selection realizes the full-data
optimal loss, and the ratio is $1$ (also when $\Lstar=0$, by the $0/0$
convention).
\end{lemma}

%% file: stmt-def-circuit.tex
\begin{definition}[Positive circuit]\label{def:circuit}
For a family of vectors $(a_i)_{i\in J}$, a nonempty set $C\subseteq J$ is a
\emph{positive circuit} if there are coefficients $\lambda_i>0$ ($i\in C$)
with $\sum_{i\in C}\lambda_ia_i=0$ and $\dim\ker[a_i]_{i\in C}=1$.
Equivalently, the positive dependence is unique up to scaling and no proper
subset of $C$ supports a nonzero nonnegative dependence. The positive
circuits are exactly the supports of the extreme rays of the cone
$\{\lambda\ge0:\sum_i\lambda_ia_i=0\}$. Every nonzero point of this cone
is a positive combination of extreme rays (conic Carath\'eodory). A singleton
$\{i\}$ with $a_i=0$ is a positive circuit.
\end{definition}

%% file: sec-threshold.tex
\section{The exact threshold \texorpdfstring{$(m+1)d$}{(m+1)d}}\label{sec:threshold}

\paragraph{Upper bound: fixed-basis conic compression.} The key
combinatorial fact is that a zero-sum family of vectors admits a nonnegative
zero-sum subfamily that contains any prescribed index set $B$ and at most
$\dim\spn\{a_i:i\notin B\}$ further elements (Lemma~\ref{lem:compression},
a consequence of conic Carath\'eodory). Applied to the residual dyads with
$B$ a feature basis, it gives:

\input{stmt-thm-upper}

\begin{proof}[Proof sketch]
Pick a basis $B$ of $r$ features. The dyads $M_i=\rho_ix_i^\top$ have all
rows in $T$, so the dyads outside $B$ span a space of dimension at most $mr$.
Lemma~\ref{lem:compression} therefore yields $c\ge0$, positive on $B$, with
$\sum_ic_iM_i=0$ and support at most $r+mr$. The support contains the basis,
hence spans $T$, and Lemma~\ref{lem:certificate} gives $A(F_c)=W^\circ$.
The full proof is in Appendix~\ref{app:threshold}. The same argument with
$s=\dim\spn\{\rho_i\}$ replacing $m$ gives $\tau(D)\le(s+1)r$ for
$r\ge1$ (Remark~\ref{rem:effdim}): the dataset's budget is governed by its
effective residual output dimension.
\end{proof}

\paragraph{Lower bound: the axial simplex instance.} Fix $d,m\ge1$, let
$\ones_m$ be the all-ones vector, and let $v_j=e_j$ ($1\le j\le m$),
$v_{m+1}=-\ones_m$, so that $\sum_jv_j=0$ and no proper subset of
$\{v_j\}$ contains $0$ in its convex hull.

\input{stmt-def-axial}

On $D_{d,m}$ the unique full-data minimizer is $W^\circ=\ones_m\ones_d^\top$
with residual $v_j$ at $z_{ij}$ and $\Lstar=\frac{2m}{m+1}$. The weighted
normal equations decouple across the $d$ axes. The returned column on an
axis is the conditional weighted mean of its labels (or $0$ if the axis is
unused). The returned model's loss exceeds $\Lstar$ by
$\frac1d\sum_i\Wnorm{\delta_i}^2$, where $\delta_i$ is the deviation of
column $i$ from $\ones_m$ (Lemma~\ref{lem:axialbasic}). With
$Q_m=m^2+m-1$, the squared distance from the origin to the convex hull of
any nonempty proper subset of $\{v_j\}$ is at least $1/Q_m$, attained on the
facets omitting some $e_k$ (Lemma~\ref{lem:facets}).

\input{stmt-thm-axial}

\begin{proof}[Proof sketch]
A selection of support at most $N-1$ leaves some axis either empty (deviation
$\Wnorm{\ones_m}^2=m$) or supported on a proper subset of its $m+1$ points
(deviation at least $1/Q_m$ by Lemma~\ref{lem:facets}), so the excess loss is
at least $\frac1{dQ_m}$. Omitting one point $e_1$ on one axis and using the
facet-foot weights there attains this bound. Dividing by $\Lstar$ gives the value.
\end{proof}

\input{stmt-thm-threshold}

\begin{proof}
Theorem~\ref{thm:upper} gives ratio $1$ at budget $(m+1)d$ for every dataset.
Theorem~\ref{thm:axial} exhibits a dataset with ratio $>1$ at budget
$(m+1)d-1$, and $\Fw(d,m,\cdot)$ is nonincreasing (Lemma~\ref{lem:support}).
\end{proof}

\input{stmt-cor-22}

%% file: stmt-thm-upper.tex
\begin{theorem}[Data-dependent upper bound]\label{thm:upper}
Every dataset with feature rank $r\ge1$ admits a spanning zero certificate of
support at most $(m+1)r$. Hence $\tau(D)\le(m+1)r$ and
$L^{\star}_D\bigl((m+1)d;\mathrm{weighted}\bigr)=\Lstar$ for every dataset.
If $r=0$, a single point suffices.
\end{theorem}

%% file: stmt-def-axial.tex
\begin{definition}[Integer axial instance]\label{def:axial}
$D_{d,m}$ consists of the $N=(m+1)d$ points
$z_{ij}=(e_i,\,\ones_m-v_j)$ for $i\in[d]$, $j\in[m+1]$. Explicitly, the
label of $z_{ij}$ is $\ones_m-e_j$ for $j\le m$ and $2\cdot\ones_m$ for
$j=m+1$.
\end{definition}

%% file: stmt-thm-axial.tex
\begin{theorem}[Exact deficit of the axial instance]\label{thm:axial}
For every $d,m\ge1$,
\[
  \frac{L^{\star}_{D_{d,m}}\bigl((m+1)d-1;\mathrm{weighted}\bigr)}
       {L^{\star}_{D_{d,m}}}
  \;=\;1+\frac{m+1}{2md\,(m^2+m-1)}\;>\;1 .
\]
\end{theorem}

%% file: stmt-thm-threshold.tex
\begin{theorem}[Threshold]\label{thm:threshold}
$\nstar(d,m)=(m+1)d$ for all $d,m\ge2$ (and in fact for all $d,m\ge1$,
cf.\ Remark~\ref{rem:range}), and $\Fw(d,m,n)=1$ for all $n\ge(m+1)d$.
\end{theorem}

%% file: stmt-cor-22.tex
\begin{corollary}\label{cor:22}
$\nstar(2,2)=6$. On the instance $D_{2,2}$ (six integer points, labels
$(0,1),(1,0),(2,2)$ over each of $e_1,e_2$) the optimal five-point ratio is
exactly $\tfrac{43}{40}$.
\end{corollary}

%% file: sec-rigidity.tex
\section{Rigidity of maximal certificates}\label{sec:rigidity}

Theorem~\ref{thm:upper} gives $\tau(D)\le\max\{1,(m+1)r\}$. The upper bound
at budget
$(m+1)d-1$ requires understanding the datasets with $\tau(D)=(m+1)d$, which
we call \emph{certificate-maximal}. The key is a rigidity theorem for
maximal certificates.

\input{stmt-thm-rigidity}

\begin{proof}[Proof sketch]
Vectorize the dyads of $S$ into a matrix $\mathsf A$ with $md$ rows and
$k$ columns. The normalized nonnegative kernel
$P=\{\theta\ge0:\mathsf A\theta=0,\ \ones^\top\theta=1\}$ is a polytope of
dimension $h-1\ge d-1$ whose relative interior consists of the strictly
positive points. The polynomial $f(\theta)=\det(\sum_i\theta_ix_ix_i^\top)$,
of degree at most $d$, is positive on the relative interior and must vanish
at every boundary point: a boundary point with $f>0$ would be a spanning
zero certificate of smaller support. A nonzero polynomial of degree $\le d$
vanishing on every facet is divisible by the facets' affine forms, so $P$
has at most $d$ facets. A bounded $(h-1)$-polytope has at least $h$ facets,
so $h=d$ and $P$ is a simplex (a point if $d=1$). The $d$ extreme rays give
the blocks $C_j$. Dropping a ray gives a boundary point where the features
cannot span. This forces the feature spans $U_j$ of the blocks to be a
minimal spanning family, hence lines. The bound $|C_j|\le m+1$ then forces
$|C_j|=m+1$ and disjointness. The full proof is in Appendix~\ref{app:near}.
\end{proof}

Using the maximality of $\tau(D)$ once more, the same line-circuit structure
governs every nonzero-feature point:

\input{stmt-thm-lines}

%% file: stmt-thm-rigidity.tex
\begin{theorem}[Determinant--facet rigidity]\label{thm:rigidity}
Let the features span $\R^d$ and suppose $\tau(D)=k:=(m+1)d$. Let $c$ be a
spanning zero certificate with $|\supp c|=k$ and support $S$. Then $S$ splits
uniquely (up to permutation) into $d$ pairwise disjoint blocks
$S=C_1\sqcup\dots\sqcup C_d$ with $|C_j|=m+1$, where each $C_j$ supports a
positive circuit of the dyads whose features span a one-dimensional line
$U_j$, and $U_1,\dots,U_d$ are linearly independent.
\end{theorem}

%% file: stmt-thm-lines.tex
\begin{theorem}[Line structure]\label{thm:lines}
Let the features span $\R^d$ and $\tau(D)=(m+1)d=k$. Then:
\begin{enumerate}
\item every data point with $x_i\ne0$ belongs to some positive circuit of
  dyads supported on a single feature line and of size exactly $m+1$;
\item there is no positive circuit of dyads whose features span a subspace of
  dimension $\ge2$, and no line circuit of size $\le m$; in particular no
  point has $x_i\ne0$ and $\rho_i=0$;
\item the nonzero features lie on exactly $d$ linearly independent lines.
\end{enumerate}
Consequently, any dataset whose nonzero features are \emph{not} of this
$d$-line form satisfies $\tau(D)\le(m+1)d-1$ and is recovered exactly within
budget $(m+1)d-1$.
\end{theorem}

%% file: sec-near.tex
\section{The near-threshold value \texorpdfstring{$1+1/(dm^2)$}{1+1/(d m²)}}\label{sec:nearthreshold}
The assembly needs one more ingredient: a sharp sparsification bound for
zero-mean weighted point systems. The general $k$-point version is
Theorem~\ref{thm:kpoint}. The case below is elementary:

\input{stmt-lem-mmean}

\begin{proof}[Proof sketch]
Decompose $p$ into positive circuits and choose one with second
moment at most $E$. For a circuit $q_1,\dots,q_{m+1}>0$, deleting point
$j$ leaves the $m$-point mean $\mu_j=-\frac{q_j}{1-q_j}z_j$. The inequality
$\sum_j(1-q_j)^2/q_j\ge m^2$ (Cauchy--Schwarz) shows that not all $j$ can
violate $\Wnorm{\mu_j}^2\le E/m^2$. The full proof is in Appendix~\ref{app:near}.
\end{proof}

\input{stmt-thm-nearthreshold}

\begin{proof}[Proof sketch]
\emph{Upper bound.} If the feature rank is below $d$, or $\Lstar=0$, or
$\tau(D)\le(m+1)d-1$, the ratio is $1$. Otherwise
Theorem~\ref{thm:lines} places all nonzero features on $d$ independent lines
$u_1,\dots,u_d$, with $x_i=s_iu_{j(i)}$. The first-order condition splits
into one zero-mean condition per line, $\sum_{i\in I_j}s_i\rho_i=0$. Let
$z_i=\rho_i/s_i$ and $p_i=s_i^2/\sum_{I_j}s_{i'}^2$, so that
$\sum_{I_j}p_iz_i=0$ and $\sum_{I_j}p_i\Wnorm{z_i}^2=R_j/g_j$ with
$R_j=\sum_{I_j}\Wnorm{\rho_i}^2$ and $g_j=\sum_{I_j}s_i^2$. Choose the
lightest line $j_\star$ ($R_{j_\star}\le N\Lstar/d$). Fit every other line
exactly with a positive $(m+1)$-point circuit, and line $j_\star$ with the
$m$-point mean $\delta$ of Lemma~\ref{lem:mmean}. The weighted normal
equations decouple along the basis, giving $\widehat W=W^\circ+H$ with
$Hu_j=0$ for $j\ne j_\star$ and $Hu_{j_\star}=-\delta$. By
\eqref{eq:lossexp-main}, the excess loss is
$g_{j_\star}\Wnorm\delta^2/N\le R_{j_\star}/(Nm^2)\le\Lstar/(dm^2)$ on a
support of size $(d-1)(m+1)+m=(m+1)d-1$.
\emph{Lower bound.} Take $d$ axes, each carrying the vertices $w+r_j$ of a
regular simplex $r_1,\dots,r_{m+1}$ with $\Wnorm{r_j}^2=\frac m{m+1}$.
The simplex's facets are at squared distance $h^2=\frac1{m(m+1)}$ from the origin,
and we take $\Wnorm w=h$. Any selection of $N-1$ points leaves an axis empty or on a
proper subset, costing at least $h^2/d$ in excess loss. Deleting one
vertex attains this bound. The ratio is $1+\frac{h^2/d}{m/(m+1)}=1+\frac1{dm^2}$.
\end{proof}

\paragraph{Failure at $n=2d$ for $m\ge2$.} The scalar rule
``$n\ge2d$ suffices'' does not extend to vector outputs
(Remark~\ref{rem:correction}).

\input{stmt-prop-2dfails}

The proof enumerates per-axis costs and allocations
(Appendix~\ref{app:threshold}).

\begin{remark}[Scalar and vector recovery budgets]\label{rem:correction}
The weighted scalar trichotomy of \citet[Theorem~1]{HMSYnote} gives ratio
$1$ for all $n\ge2d$. This sufficient budget does not extend to general
$m$: the scalar result applies to $m=1$, whereas for every $m\ge2$,
Theorem~\ref{thm:threshold} shows that worst-case exact recovery requires
$(m+1)d>2d$ points. Quantitatively,
$\Fw(d,m,2d)\ge\Fw(d,m,(m+1)d-1)=1+\frac1{dm^2}>1$ by monotonicity.
Proposition~\ref{prop:2dfails} makes the failure concrete.
\end{remark}

%% file: stmt-lem-mmean.tex
\begin{lemma}[$m$-point mean lemma]\label{lem:mmean}
Let $z_i\in\R^m$ carry weights $p_i>0$ with $\sum_ip_i=1$ and
$\sum_ip_iz_i=0$, and let
$E=\sum_ip_i\Wnorm{z_i}^2$. Then some convex combination $\delta$ of at most
$m$ of the points satisfies $\Wnorm\delta^2\le E/m^2$. The constant is
attained by the uniformly weighted regular simplex.
\end{lemma}

%% file: stmt-thm-nearthreshold.tex
\begin{theorem}[Near-threshold value]\label{thm:nearthreshold}
For all $d,m\ge2$ (and in fact for all $d,m\ge1$, cf.\
Remark~\ref{rem:range}),
$\Fw\bigl(d,m,(m+1)d-1\bigr)=1+\dfrac1{dm^2}$.
\end{theorem}

%% file: stmt-prop-2dfails.tex
\begin{proposition}\label{prop:2dfails}
For all $d\ge1,m\ge2$,
$\Fw(d,m,2d)\ \ge\ \Fw\bigl(d,m,(m+1)d-1\bigr)\;=\;1+\frac1{dm^2}>1$,
since $2d\le(m+1)d-1$ and $\Fw$ is nonincreasing in $n$. Concretely, for
$d=m=2$ the instance $D_{2,2}$ of Corollary~\ref{cor:22} satisfies: every
weighted selection of at most $2d=4$ points has
\[
  L_D(A(F))\;\ge\;\frac{23}{20}\cdot\Lstar ,
\]
with equality attained (allocate two points per axis, at conditional weights
realizing the nearest facet points).
\end{proposition}

%% file: sec-remaining.tex
\section{Further results on the selection profile}\label{sec:remaining}

\input{sec-lowbudget}
\input{sec-cell}

%% file: sec-lowbudget.tex
\subsection{Low budgets \texorpdfstring{$n\le d$}{n<=d}}\label{sec:lowbudget}

\input{stmt-thm-lowbudget}

The lower bounds are inherited from the scalar case through
Lemma~\ref{lem:embed}: by \citet[Theorem~1]{HMSYnote}, $\Fsc(d,n)=\infty$
for $n<d$ and $\Fsc(d,d)=d+1$. The content of the theorem is the upper bound
$\Fw(d,m,d)\le d+1$ \emph{uniformly in $m$}, which requires one
subset serving all $m$ output rows simultaneously.

\begin{proof}[Proof sketch]
Let $X\in\R^{d\times N}$ have full row rank and $y\in\R^N$. Size-$d$ volume
sampling draws $d$ columns with probability proportional to $\det(X_S)^2$,
and $w(S)$ is the interpolant of the subproblem. \citet{DW17} proved that
$\E[L(w(S))]\le(d+1)L(w^\star)$ for the total square loss, with equality
when $X$ is in general position (their Theorem~5). We use only the
inequality, which holds without general position. Equality can fail for
degenerate $X$; see Appendix~\ref{app:lowbudget}. The distribution over subsets
depends only on the features. For a dataset of feature rank $r$ (the cases
$r=0$ and $\Lstar=0$ are trivial), we run size-$r$ volume sampling in
coordinates of $T$, apply the inequality to each of the $m$ output rows, and
sum: $\E[L_D(W_S)]\le(r+1)\Lstar\le(d+1)\Lstar$ for the matrix $W_S$ whose
rows are the per-row interpolants. Some subset $S$ achieves the bound, and
$W_S$ is exactly what the Frobenius rule returns on $S$ with any positive
weights, since the $r$ selected features are linearly independent.
\end{proof}

%% file: stmt-thm-lowbudget.tex
\begin{theorem}\label{thm:lowbudget}
For all $d,m\ge2$ (and in fact for all $d,m\ge1$, cf.\
Remark~\ref{rem:range}),
\[
  \Fw(d,m,n)=\infty\ \text{ for every } n<d,
  \qquad\text{and}\qquad
  \Fw(d,m,d)=d+1 .
\]
\end{theorem}

%% file: sec-cell.tex
\subsection{The smallest intermediate cell \texorpdfstring{$(d,m)=(2,2)$}{(2,2)}}\label{sec:intermediate}

For $d,m\ge2$ and $d<n<(m+1)d-1$ the profile remains open (for $d=1$ it is
given by Corollary~\ref{cor:d1}). Apart from the cells settled by
\citet{ZhangScalar}, its scalar counterpart also remains open. We treat the
smallest cell $(d,m)=(2,2)$, where the open budgets are $n=3,4$.

\paragraph{A complex dictionary for the hard branches.} Fix a $(2,2)$
dataset of feature rank $2$ with $\Lstar>0$ and $\tau(D)>n$. For selections
whose features span $\R^2$, the weighted minimizer is unique. Its
full-data loss is invariant under invertible linear changes of features, so
we may assume the data are whitened, $\sum_ix_ix_i^\top=\gamma I_2$. Identify
$\R^2\cong\C$, write $\xi_i$ for the feature and $\eta_i$ for the residual
$\rho_i$ of point $i$, and for $\xi_i\ne0$ set
\[
  p_i=\frac{|\xi_i|^2}{\sum_j|\xi_j|^2},\qquad
  t_i=\frac{\bar\xi_i}{\xi_i}\in S^1,\qquad
  z_i=\frac{\eta_i}{\xi_i}\in\C,\qquad
  E=\sum_ip_i|z_i|^2>0 .
\]
Whitening and the first-order condition become the three moment conditions
$\E_pt=\E_pz=\E_p(\bar tz)=0$. Conversely, every finite \emph{moment
system} $(p_i,t_i,z_i)$ satisfying them is realized by a whitened dataset. A
selection with weights $c$ corresponds to $q_i\propto c_i|\xi_i|^2$. With
$\mu=\E_qt$, $\nu=\E_qz$, $\lambda=\E_q(\bar tz)$, the map $W^\circ-\widehat
W$ is $\xi\mapsto a\xi+b\bar\xi$ where
\begin{equation}\label{eq:ab-main}
  \begin{pmatrix}1&\mu\\ \bar\mu&1\end{pmatrix}
  \begin{pmatrix}a\\b\end{pmatrix}
  =\begin{pmatrix}\nu\\\lambda\end{pmatrix},
  \qquad |\mu|<1\iff\text{selected features span }\R^2,
\end{equation}
and the excess ratio of the selection is at most $C(q)/E$ with
$C(q)=|a|^2+|b|^2$ (with equality when no feature vanishes;
Appendix~\ref{subsec:dictionary}). For $k\in\{3,4\}$ let $\Gamma_k$ be the
supremum over moment systems of
$\frac1E\inf\{C(q):|\supp q|\le k,\ |\mu(q)|<1\}$. The ill-conditioning factor $(1-|\mu|^2)^{-1}$ in
\eqref{eq:ab-main} disappears when $(a,b)$ is written in terms of the
two-point interpolation coefficients:

\input{stmt-lem-interp}

Consequently, for a support $S$ with $2\le|S|\le3$ containing two distinct
directions, the infimum of $C(q)$ over admissible $q$ supported in $S$
equals the squared distance from the origin to
$K_S=\conv\{B_{ij}:i,j\in S,\ t_i\ne t_j\}$ (Lemma~\ref{lem:closure}). Two further identities follow. The star of an atom
$i$, i.e.\ the coefficients $B_{ij}$ over $j$, lies in a real two-dimensional
affine plane. The star has $P_{ij}=\frac12p_jd_{ij}$-weighted mean $\beta_i$ with
$\Wnorm{\beta_i}^2=\frac12|z_i|^2$ and variance $E+\frac{|z_i|^2}2-L_i$
(Lemma~\ref{lem:star}). Applying Lemma~\ref{lem:mmean} in that plane and
choosing $i$ with $|z_i|^2\le E$ gives
$\inf_{|\supp q|\le3,|\mu|<1}C(q)\le\frac{7E}8$. Letting the weight
concentrate on atom $i$ gives $\inf_{|\supp q|\le4,|\mu|<1}C(q)\le\frac E2$
(Lemma~\ref{lem:ehalf}). Hence $\Gamma_3\le\frac78$ and
$\Gamma_4\le\frac12$.

\paragraph{Lower bounds.}

\input{stmt-prop-lower22}

The instance is axial: two axes, each carrying the vertices $w+r_j$ of a
centered equilateral triangle ($\Wnorm{r_j}^2=\frac23$, $\Wnorm w^2=\frac56$).
Minimizing the per-axis costs over allocations gives the two values
(Appendix~\ref{app:cell}).

\input{stmt-thm-intervals}

\begin{proof}[Proof sketch]
The lower bounds follow from Proposition~\ref{prop:lower22}. For the upper bounds, the
branches of feature rank $\le1$, $\Lstar=0$, or $\tau(D)\le n$ have ratio
$1$ (Theorem~\ref{thm:upper}). Otherwise, the dictionary applies and the ratio
is at most $1+\Gamma_3\le\frac{15}8$ for $n=3$ and $1+\Gamma_4\le\frac32$
for $n=4$. The closure points are limits of admissible selections
within the same budget. This suffices for the inner
infimum. The full proofs are in Appendix~\ref{app:cell}.
\end{proof}

\paragraph{The conjecture and its reduction.}

\input{stmt-conj-CE}

The two formulations are equivalent (Remark~\ref{rem:equiv}), and the
conjecture is a finite-dimensional problem:

\input{stmt-prop-seven}

\paragraph{Evidence.} Appendix~\ref{sec:evidence} proves four further
statements. \textbf{(i)}~\emph{Two-direction systems} (all atoms on two antipodal
directions) satisfy both conjectured inequalities. Both constants are
attained simultaneously by six equally weighted atoms forming an
equilateral $z$-triangle over each direction. For the three-point bound,
this is, up to merging coincident atoms, the only equality configuration
(Proposition~\ref{prop:twodir}).
\textbf{(ii)}~The equality fiber is a stratified local maximum of the normalized
three-point value $E^{-1}\inf C$: energy imbalance gives the cusp
$\frac18(5-3|\eta|)$ below its value $\frac58$, and splitting a direction
class drops it to $\limsup\le\frac12$ (Proposition~\ref{prop:stability}). \textbf{(iii)}~Two exact obstructions: a proof
cannot restrict to selections with $\mu=0$ (for the seventh-roots
system $t_j=\zeta^j$, $z_j=\zeta^{3j}$, Parseval forces $C\ge\frac38,\frac23$
at supports $4,3$, while unconstrained selections reach $<\frac14$ and
$\frac27$). It also cannot rest on comparing aggregate second moments of original-weight
subsets (for the fifth-roots system the aggregate quotient is
$\frac{15}{23}>\frac58$) (Propositions~\ref{prop:mu0}--\ref{prop:averaging}).
\textbf{(iv)}~Closed classes:

\input{stmt-prop-classes}

Moreover, by Lemma~\ref{lem:star} any counterexample to the three-point
inequality must satisfy $5|z_i|^2-2L_i>3E$ for every atom. The regular
hexagon third-harmonic system has three-point value exactly $\frac12$ and
four-point value $0$.

%% file: stmt-lem-interp.tex
\begin{lemma}[Interpolation convex combination]\label{lem:interp}
For $t_i\ne t_j$ let
$B_{ij}=\Bigl(\frac{t_iz_j-t_jz_i}{t_i-t_j},\ \frac{z_i-z_j}{t_i-t_j}\Bigr)
\in\C^2$ (the coefficients of the affine function of $t$ through the two
points) and $d_{ij}=|t_i-t_j|^2$. Then for every admissible $q$,
\[
  \Delta(q):=1-|\mu|^2=\sum_{i<j}q_iq_jd_{ij},
  \qquad
  (a,b)=\sum_{\substack{i<j\\ t_i\ne t_j}}
  \frac{q_iq_jd_{ij}}{\Delta(q)}\;B_{ij} .
\]
In particular, $(a,b)$ is a convex combination of the pairwise interpolation
coefficients. The conditioning factor $(1-|\mu|^2)^{-1}$ is absorbed
entirely.
\end{lemma}

%% file: stmt-prop-lower22.tex
\begin{proposition}\label{prop:lower22}
There is an explicit dataset $D^\triangle$ with
\[
  \frac{L^{\star}_{D^\triangle}(3;\mathrm{weighted})}
       {L^{\star}_{D^\triangle}}=\frac{13}8,
  \qquad
  \frac{L^{\star}_{D^\triangle}(4;\mathrm{weighted})}
       {L^{\star}_{D^\triangle}}=\frac54 .
\]
Hence $\Fw(2,2,3)\ge\frac{13}8$ and $\Fw(2,2,4)\ge\frac54$.
\end{proposition}

%% file: stmt-thm-intervals.tex
\begin{theorem}\label{thm:intervals}
$\Fw(2,2,3)\in\bigl[\frac{13}8,\frac{15}8\bigr]$ and
$\Fw(2,2,4)\in\bigl[\frac54,\frac32\bigr]$.
\end{theorem}

%% file: stmt-conj-CE.tex
\begin{conjecture}\label{conj:CE}
$\Fw(2,2,3)=\frac{13}8$ and $\Fw(2,2,4)=\frac54$. Equivalently, for every
moment system,
\[
  \inf_{|\supp q|\le3,\ |\mu|<1}C(q)\le\frac{5E}8
  \qquad\text{and}\qquad
  \inf_{|\supp q|\le4,\ |\mu|<1}C(q)\le\frac E4 .
\]
\end{conjecture}

%% file: stmt-prop-seven.tex
\begin{proposition}\label{prop:seven}
For fixed atoms $(t_i,z_i)$, the closed set
$\mathcal P=\{p\in\R^N_{\ge0}:\sum_ip_i=1,\ \E_pt=\E_pz=\E_p(\bar tz)=0\}$
is a polytope cut out by at most $7$ linear equalities. Its points with
full support are the admissible weight vectors. Minimizing the linear
functional $E(p)$ over this polytope yields a basic feasible solution
$p^\star$ of support at most $7$ with $E(p^\star)\le E(p)$. The subsystem on
$\supp p^\star$ is again a moment system, and every selection admissible for
it is admissible for the original system. Consequently each inequality of
Conjecture~\ref{conj:CE} holds
for all finite systems if and only if it holds for all systems with at most
seven atoms. Since atoms may be split into identically placed copies
without changing anything, it suffices to treat systems with \emph{exactly}
seven labelled atoms.
\end{proposition}

%% file: stmt-prop-classes.tex
\begin{proposition}[Classes on which Conjecture~\ref{conj:CE} holds]
\label{prop:classes}
The three-point inequality of Conjecture~\ref{conj:CE} holds for all moment
systems with at most four atoms. The four-point inequality holds for all
systems with at most five atoms. Both inequalities hold for all two-direction
systems, and the four-point inequality holds for all antipodally paired
systems (disjoint unions of pairs $\{(t,z),(-t,z)\}$ with equal $z$-values
inside each pair).
\end{proposition}

%% file: sec-discussion.tex
\section{Discussion and open problems}\label{sec:discussion}

The threshold splits conceptually as $(m+1)d=d+md$: $d$ points to pin a
feature basis and $md$ points to cancel the $m\times d$ residual gradient.
The lower-bound instances realize this as $d$ independent copies of weighted
mean estimation in $\R^m$, one per feature direction. Each requires its
full simplex of $m+1$ points. At $m=1$ the two mechanisms coincide.
Table~\ref{tab:instances} in Appendix~\ref{app:cell} separates the values of
explicit instances from the global profile values at $(2,2)$. Our theorems
concern a fixed learner and a fixed finite dataset. Certificates use
the full-data minimizer $W^\circ$, with arbitrary nonnegative
real weights. Some intermediate bounds are attained only as limits. The theorems
describe what a small weighted subsample can represent and its exact
budget limits. They provide neither a faster selection algorithm nor
guarantees for unseen data.

\paragraph{Open problems.}
(1)~\emph{The intermediate curve.} Determine $\Fw(d,m,n)$ for $d,m\ge2$ and
$d<n<(m+1)d-1$. Even the scalar case is largely open
(\citet{HMSYnote}, Question~2; see \citet{ZhangScalar} for the cells
resolved so far). Axial instances reduce to an allocation problem over
per-line mean-estimation deficits (Theorem~\ref{thm:kpoint} supplies the
sharp per-line constants).
(2)~\emph{The $(2,2)$ cell} (Conjecture~\ref{conj:CE}): a question about at
most seven atoms on the circle. A proof must use selections with $\mu\ne0$
and cannot rest on comparing aggregate second moments of original-weight
subsets.
(3)~\emph{Unweighted vector selection.} Relating the unweighted profile
$F(d,m,n)$ of Question~3 of \citet{HMSYnote} to the weighted threshold
$(m+1)d$.
(4)~\emph{Other learning rules.} How does the threshold depend on the ERM
tie-breaking rule? Our upper bounds use it only through
Lemma~\ref{lem:certificate}, while the lower bounds use it in their
rank-deficient branches.

\subsection{Conclusion}\label{sec:conclusion}

We establish the exact threshold $\nstar(d,m)=(m+1)d$ for weighted
data selection with the minimum-Frobenius-norm ERM to recover the full-data
loss of every dataset. This answers the threshold part of Question~4 posed at COLT~2025 by \citet{HMSYnote}. We also determine the
profile at two further budgets, $\Fw(d,m,(m+1)d-1)=1+1/(dm^2)$ and
$\Fw(d,m,d)=d+1$. The proofs rest on three reusable tools: fixed-basis conic
compression, determinant--facet rigidity of maximal certificates, and sharp
sparsification of zero-mean weighted point systems. The proofs also use a
complex-variable dictionary that reduces the smallest open cell to a
seven-atom moment problem. The intermediate curve
$d<n<(m+1)d-1$ for $d,m\ge2$ remains open. For $d=1$ the complete curve is
given by Corollary~\ref{cor:d1}. Even on the open curve the inner
infimum in $\Fw$ need not be attained (Remark~\ref{rem:nonattain}: an
integer instance with $d=m=2$, $n=3$). The first target is the $(2,2)$
conjecture $\Fw(2,2,3)=13/8$, $\Fw(2,2,4)=5/4$.

%% file: app-prelim.tex
\section{Proofs for Section~\ref{sec:prelim}}\label{app:prelim}

This appendix proves the structural lemmas of Section~\ref{sec:prelim}.

\subsection{Row decomposition and minimum-norm geometry}

\begin{lemma}[Row decomposition]\label{lem:rows}
Write $w_j^\top$ for the $j$-th row of $W$ and $y_{ij}$ for the $j$-th entry
of $y_i$. Then
$F_c(W)=\sum_{j=1}^m\sum_ic_i(\langle w_j,x_i\rangle-y_{ij})^2$ and
$\Wnorm{W}_F^2=\sum_j\Wnorm{w_j}^2$. Hence the minimizer set of $F_c$ is the
Cartesian product over rows of scalar weighted least-squares solution sets.
The model $A(F_c)$ consists of the minimal-$\ell_2$-norm scalar solutions row by
row, \emph{with the single shared weight vector $c$}.
\end{lemma}

\begin{proof}
Both the objective and the squared Frobenius norm are additive across rows,
and a product set is minimized in norm coordinate-wise.
\end{proof}

Let $T=\spn\{x_1,\dots,x_N\}$, $r=\dim T$, and let $P_T$ denote the
orthogonal projection onto $T$. Let $W^\circ$ be the minimal-Frobenius-norm
full-data minimizer, and set $\rho_i=W^\circ x_i-y_i$ (residuals) and
$M_i=\rho_ix_i^\top\in\R^{m\times d}$ (residual dyads).

\begin{lemma}[Full-data geometry]\label{lem:geometry}
$W^\circ=W^\circ P_T$. The first-order condition
$\sum_{i=1}^NM_i=0$ holds, and for every $H\in\R^{m\times d}$,
\begin{equation}\label{eq:lossexp}
  L_D(W^\circ+H)=\Lstar+\frac1N\sum_{i=1}^N\Wnorm{Hx_i}^2 .
\end{equation}
Consequently, the full-data minimizer set is
$\{W^\circ+H: HP_T=0\}$. If $T=\R^d$, the minimizer is unique.
\end{lemma}

\begin{proof}
Predictions depend on $W$ only through $WP_T$, and
$\Wnorm{W}_F^2=\Wnorm{WP_T}_F^2+\Wnorm{WP_{T^\perp}}_F^2$, so minimality of
the norm forces $W^\circ P_{T^\perp}=0$. The condition $\sum_iM_i=0$ is the
vanishing gradient of the smooth convex $L_D$ at $W^\circ$. Expanding
$L_D(W^\circ+H)$, the cross term is
$\frac2N\langle H,\sum_iM_i\rangle_F=0$, giving \eqref{eq:lossexp}. The
excess term in \eqref{eq:lossexp} vanishes iff $Hx_i=0$ for all $i$, i.e.\
$HP_T=0$, which identifies the minimizer set.
\end{proof}

The two conditions in the next lemma have separate roles: the zero-dyad
condition makes $W^\circ$ stationary for the selected objective, while the
spanning condition makes the selected and full-data minimizer sets coincide.

\begin{restate}{Lemma~\ref{lem:certificate} (Exact certificate; restated)}
Let $c\in\R^N_{\ge0}$ be nonzero, with support $S$ (so $c_i>0$ on $S$). If
\begin{equation*}
  \sum_{i\in S}c_iM_i=0
  \qquad\text{and}\qquad
  \spn\{x_i:i\in S\}=T ,
\end{equation*}
then $A(F_c)=W^\circ$. In particular, the selection realizes the full-data
optimal loss, and the ratio is $1$ (also when $\Lstar=0$, by the $0/0$
convention).
\end{restate}

\begin{proof}
Using \eqref{eq:cert}, for any $H$,
$F_c(W^\circ+H)-F_c(W^\circ)
 =2\langle H,\sum_{i\in S}c_iM_i\rangle_F+\sum_{i\in S}c_i\Wnorm{Hx_i}^2
 =\sum_{i\in S}c_i\Wnorm{Hx_i}^2$,
which vanishes iff $Hx_i=0$ on $S$, iff $HP_T=0$ because the support spans $T$.
Thus the weighted minimizer set is $W^\circ+\{H:HP_T=0\}$, the same set as the
full-data minimizer set. Since $W^\circ=W^\circ P_T$ and $H=HP_{T^\perp}$ are
Frobenius-orthogonal, the unique minimal-norm element is $W^\circ$.
\end{proof}

\begin{lemma}[Scalar embedding]\label{lem:embed}
Let $D_0=((x_i,t_i))_{i=1}^N$ be a scalar dataset ($x_i\in\R^d$,
$t_i\in\R$) and let $\widetilde D=((x_i,t_i\ones_m))_{i=1}^N$ be its copy in
which all $m$ coordinates of every label agree. For every weight vector $c$, all rows
of $A(F_c)$ computed on $\widetilde D$ equal the scalar minimum-norm weighted
ERM of $D_0$ with the same weights. For every $W$ whose rows all equal a
common $w\in\R^d$, one has $L_{\widetilde D}(W)=m\,L_{D_0}(w)$. Consequently,
$\Fw(d,m,n)\ \ge\ \Fsc(d,n)$ for all $n$, where $\Fsc(d,n)$ denotes the
scalar weighted profile of \citep{HMSYnote}.
\end{lemma}

\begin{proof}
By Lemma~\ref{lem:rows}, every row of $A(F_c)$ is the minimum-norm solution
of one and the same scalar weighted least-squares problem, so the rows
agree. The identity $L_{\widetilde D}(W)=mL_{D_0}(w)$ is immediate from the
definition of the loss. Applying this to the selected model and to the
full-data minimizer (uniform weights), numerator and denominator of the
ratio both scale by $m$, so every scalar instance embeds with the same
ratio.
\end{proof}

%% file: app-threshold.tex
\section{Proofs for Section~\ref{sec:threshold} and Proposition~\ref{prop:2dfails}}\label{app:threshold}

\subsection{Upper bound: fixed-basis conic compression}\label{loc:proof-upper}

\begin{lemma}[Fixed-basis conic compression]\label{lem:compression}
Let $a_1,\dots,a_N$ be elements of a finite-dimensional real vector space
with $\sum_{i=1}^Na_i=0$, and let $B\subseteq[N]$. Put
$q_B=\dim\spn\{a_i:i\notin B\}$. Then there exist coefficients $c_i\ge0$ with
$c_b>0$ for all $b\in B$, $\sum_ic_ia_i=0$, and
$|\supp c|\le|B|+q_B$.
\end{lemma}

\begin{proof}
The global zero sum gives $-\sum_{b\in B}a_b=\sum_{i\notin B}a_i$, which is a
nonnegative combination of $\{a_i\}_{i\notin B}$ and hence lies in
$\cone\{a_i:i\notin B\}$. By conic Carath\'eodory (see e.g.\
\citep[Prop.~1.2.1(a)]{Bertsekas}), every element of this cone is a positive
combination of a \emph{linearly independent} subfamily, so
$-\sum_{b\in B}a_b=\sum_{i\in C}\alpha_ia_i$ with $\alpha_i>0$,
$C\subseteq[N]\setminus B$ and $|C|\le q_B$. If the target is $0$, take
$C=\varnothing$. Set $c_b=1$ on $B$, $c_i=\alpha_i$ on $C$, and $0$
elsewhere.
\end{proof}

\begin{restate}{Theorem~\ref{thm:upper} (Data-dependent upper bound; restated)}
Every dataset with feature rank $r\ge1$ admits a spanning zero certificate of
support at most $(m+1)r$. Hence $\tau(D)\le(m+1)r$ and
$L^{\star}_D\bigl((m+1)d;\mathrm{weighted}\bigr)=\Lstar$ for every dataset.
If $r=0$, a single point suffices.
\end{restate}

\begin{proof}
Choose $B\subseteq[N]$ with $|B|=r$ such that $\{x_b\}_{b\in B}$ is a basis of
$T$; this is possible because the features lie in and span $T$. Apply
Lemma~\ref{lem:compression} to the residual dyads $a_i=M_i$, whose global sum
vanishes by Lemma~\ref{lem:geometry}. Every $M_i=\rho_ix_i^\top$ has all rows
proportional to $x_i^\top\in T$, so
$\spn\{M_i:i\notin B\}\subseteq\R^m\otimes T$, and thus $q_B\le mr$. The
resulting $c$ is strictly positive on a set $S\supseteq B$ of size at most
$r+mr=(m+1)r$ and satisfies $\sum_{i\in S}c_iM_i=0$ and
$\spn\{x_i:i\in S\}\supseteq\spn\{x_b:b\in B\}=T$. Since all features lie in
$T$, the span equals $T$. Normalize $c$ to sum $1$; its support is
nonempty. Lemma~\ref{lem:certificate} then gives $A(F_c)=W^\circ$,
so the selected model attains the optimal full-data loss. By
Lemma~\ref{lem:support}, the support fits in $(m+1)d\ge (m+1)r$ slots. If
$r=0$, then all $x_i=0$. In this case, every $W$ has the same loss on every objective, and
$A$ returns $W=0$ in both cases, so one point suffices.
\end{proof}

\begin{remark}[Effective residual dimension]\label{rem:effdim}
The proof used only that the residual dyads lie in $\R^m\otimes T$. If
$s=\dim\spn\{\rho_1,\dots,\rho_N\}$, they lie in the $sr$-dimensional space
$\spn\{\rho_i\}\otimes T$, and the same argument gives
$\tau(D)\le r+sr=(s+1)r$ for $r\ge1$. The budget needed to recover a given
dataset exactly is governed by its effective residual output dimension
$s\le m$, not by the nominal output dimension.
\end{remark}

\subsection{Lower bound: the axial simplex instance}\label{loc:proof-axial}

Fix $d,m\ge1$ and let $\ones_m\in\R^m$ denote the all-ones vector,
\[
  v_j=e_j\ (1\le j\le m),\qquad v_{m+1}=-\ones_m ,
\]
so $\sum_{j=1}^{m+1}v_j=0$. The unique linear dependence among
$v_1,\dots,v_{m+1}$ has all coefficients equal. In particular, no proper
subset of $\{v_j\}$ admits a nonzero nonnegative zero-sum; equivalently,
$0\notin\conv$ of any proper subset.

\begin{restate}{Definition~\ref{def:axial} (Integer axial instance; restated)}
$D_{d,m}$ consists of the $N=(m+1)d$ points
$z_{ij}=(e_i,\,\ones_m-v_j)$ for $i\in[d]$, $j\in[m+1]$. Explicitly, the
label of $z_{ij}$ is $\ones_m-e_j$ for $j\le m$ and $2\cdot\ones_m$ for
$j=m+1$.
\end{restate}

\begin{lemma}\label{lem:axialbasic}
For $D_{d,m}$, the unique full-data minimizer is $W^\circ=\ones_m\ones_d^\top$,
with residual $v_j$ at $z_{ij}$ and
$\Lstar=\frac{2m}{m+1}$. Moreover, for any weight vector $c$ with per-axis
totals $t_i=\sum_jc_{ij}$, the returned model $\widehat W=A(F_c)$ has the following form:
column $i$ of $\widehat W$ equals the conditional weighted label mean
$\ones_m-a_i$ with $a_i=\sum_j\frac{c_{ij}}{t_i}v_j$ when $t_i>0$, and equals $0$
when $t_i=0$. Writing $\delta_i=\widehat We_i-\ones_m$, we have
\begin{equation}\label{eq:axialloss}
  L_D(\widehat W)=\Lstar+\frac1d\sum_{i=1}^d\Wnorm{\delta_i}^2 .
\end{equation}
\end{lemma}

\begin{proof}
The full Gram matrix is $\sum_{ij}e_ie_i^\top=(m+1)I\succ0$ and
$\sum_{ij}v_je_i^\top=0$, so $W^\circ$ is the unique minimizer, with
$\Lstar=\frac{d\sum_j\Wnorm{v_j}^2}{N}=\frac{2md}{(m+1)d}$. Since the
features are standard basis vectors, the weighted normal equations decouple
across columns. A positive-weight column returns the conditional mean, and a
zero-weight column is unconstrained. The Frobenius rule sets the latter to
zero because the norm decomposes column-wise. For \eqref{eq:axialloss}, expand
$\Wnorm{\delta_i+v_j}^2$ and use $\sum_jv_j=0$:
$\sum_j\Wnorm{v_j+\delta_i}^2=2m+(m+1)\Wnorm{\delta_i}^2$.
\end{proof}

\begin{lemma}[Facet distances]\label{lem:facets}
Let $Q_m=m^2+m-1$. The squared distance from the origin to the convex hull of
any nonempty proper subset of $\{v_1,\dots,v_{m+1}\}$ is at least $1/Q_m$. The minimum
is attained on the facets omitting some $e_k$, while the facet omitting
$v_{m+1}$ has squared distance $1/m$.
\end{lemma}

\begin{proof}
Every proper subset is contained in a facet. For the facet
$\conv\{e_1,\dots,e_m\}$, the nearest point is $\ones_m/m$, of squared norm $1/m$.
For the facet omitting $e_k$, let $n_k$ have $-m$ in coordinate $k$ and $1$
elsewhere. The facet lies on the hyperplane $n_k^\top x=1$ and
$\Wnorm{n_k}^2=Q_m$, so the distance to the hyperplane is $1/\sqrt{Q_m}$. The
foot $n_k/Q_m$ lies in the facet because
$n_k/Q_m=\frac m{Q_m}(-\ones_m)+\sum_{j\ne k}\frac{m+1}{Q_m}e_j$ is a convex
combination. Finally, $Q_m\ge m$ with equality only at $m=1$, where both
distances equal $1=1/Q_1$.
\end{proof}

\begin{restate}{Theorem~\ref{thm:axial} (Exact deficit of the axial instance; restated)}
For every $d,m\ge1$,
\[
  \frac{L^{\star}_{D_{d,m}}\bigl((m+1)d-1;\mathrm{weighted}\bigr)}
       {L^{\star}_{D_{d,m}}}
  \;=\;1+\frac{m+1}{2md\,(m^2+m-1)}\;>\;1 .
\]
\end{restate}

\begin{proof}
A selection of support at most $N-1$ either misses some axis entirely,
giving $\Wnorm{\delta_i}^2=\Wnorm{\ones_m}^2=m\ge 1/Q_m$, or uses a proper subset of the
$m+1$ points of some axis, in which case $\delta_i=-a_i$ lies in minus the
convex hull of a proper subset of $\{v_j\}$ and
$\Wnorm{\delta_i}^2\ge1/Q_m$ by Lemma~\ref{lem:facets}. With
\eqref{eq:axialloss}, every such selection has
$L_D\ge\Lstar+\frac1{dQ_m}$. Conversely, omit $z_{1,1}$ and give axis $1$ the
conditional weights of the foot $n_1/Q_m$: put $\frac{m}{Q_m}$ on
$v_{m+1}$ and $\frac{m+1}{Q_m}$ on each $v_j$, $2\le j\le m$, and scale by
$t_1=\frac1d$. Give every other axis uniform weights. The support is exactly
$N-1$ and $L_D=\Lstar+\frac1{dQ_m}$. Dividing by $\Lstar=\frac{2m}{m+1}$
gives the claim.
\end{proof}

\begin{restate}{Theorem~\ref{thm:threshold} (Threshold; restated)}
$\nstar(d,m)=(m+1)d$ for all $d,m\ge2$ (and in fact for all $d,m\ge1$,
cf.\ Remark~\ref{rem:range}), and $\Fw(d,m,n)=1$ for all $n\ge(m+1)d$.
\end{restate}

\begin{proof}
Theorem~\ref{thm:upper} gives ratio $1$ at budget $(m+1)d$ for every dataset,
and Theorem~\ref{thm:axial} exhibits a dataset with ratio $>1$ at budget
$(m+1)d-1$. Since $\Fw(d,m,\cdot)$ is nonincreasing (Lemma~\ref{lem:support}),
the value is $1$ at every budget $n\ge(m+1)d$ and exceeds $1$ at every
budget $n\le(m+1)d-1$.
\end{proof}

\begin{restate}{Corollary~\ref{cor:22} (restated)}
$\nstar(2,2)=6$. On the instance $D_{2,2}$ (six integer points, labels
$(0,1),(1,0),(2,2)$ over each of $e_1,e_2$) the optimal five-point ratio is
exactly $\tfrac{43}{40}$.
\end{restate}

We emphasize that $43/40$ is the exact optimum \emph{of this instance}. The
worst case over all datasets at budget $5$ is
$\Fw(2,2,5)=1+\tfrac18=\tfrac98$ by Theorem~\ref{thm:nearthreshold}.

\subsection{Failure at \texorpdfstring{$n=2d$}{n=2d} for
\texorpdfstring{$m\ge2$}{m>=2}}\label{subsec:2dfails}

The scalar rule ``$n\ge2d$ suffices'' does not extend to vector outputs
(Remark~\ref{rem:correction}).

\begin{restate}{Proposition~\ref{prop:2dfails} (restated)}
For all $d\ge1,m\ge2$,
$\Fw(d,m,2d)\ \ge\ \Fw\bigl(d,m,(m+1)d-1\bigr)\;=\;1+\frac1{dm^2}>1$,
since $2d\le(m+1)d-1$ and $\Fw$ is nonincreasing in $n$. Concretely, for
$d=m=2$ the instance $D_{2,2}$ of Corollary~\ref{cor:22} satisfies: every
weighted selection of at most $2d=4$ points has
\[
  L_D(A(F))\;\ge\;\frac{23}{20}\cdot\Lstar ,
\]
with equality attained (allocate two points per axis, at conditional weights
realizing the nearest facet points).
\end{restate}

\begin{proof}
The monotonicity statement is immediate from the definitions. For $D_{2,2}$,
\eqref{eq:axialloss} and per-axis decoupling show that a budget-$4$ selection
allocates $(k_1,k_2)$ points to the two axes with $k_1+k_2\le4$. The minimal
per-axis costs are $0$ (all three points), $1/5$ (best two points, the facet
$\{e_2,-\ones_m\}$ or $\{e_1,-\ones_m\}$ at squared distance $1/Q_2=1/5$), $1$ (best
single point), and $2$ (empty axis). The minimum total cost over allocations
is $2\cdot\frac15$ at $(2,2)$, giving
$L_D=\frac43+\frac12\cdot\frac25=\frac{23}{15}$ and ratio
$\frac{23/15}{4/3}=\frac{23}{20}$.
\end{proof}

%% file: app-lowbudget.tex
\section{Proof of Theorem~\ref{thm:lowbudget}}\label{app:lowbudget}

\begin{restate}{Theorem~\ref{thm:lowbudget} (restated)}
For all $d,m\ge2$ (and in fact for all $d,m\ge1$, cf.\
Remark~\ref{rem:range}),
\[
  \Fw(d,m,n)=\infty\ \text{ for every } n<d,
  \qquad\text{and}\qquad
  \Fw(d,m,d)=d+1 .
\]
\end{restate}

The lower bounds are inherited from the scalar case through
Lemma~\ref{lem:embed}: by \citep[Theorem~1]{HMSYnote}, $\Fsc(d,n)=\infty$ for
$n<d$ and $\Fsc(d,d)=d+1$, so $\Fw(d,m,n)\ge \Fsc(d,n)$. The content of this
section is the upper bound $\Fw(d,m,d)\le d+1$ \emph{uniformly in $m$}, which
requires selecting one subset that serves all $m$ output rows
simultaneously.

\paragraph{Volume sampling.}
Let $X\in\R^{d\times N}$ have full row rank and let $y\in\R^N$. Size-$d$
volume sampling draws a subset $S$ of $d$ column indices with probability
proportional to $\det(X_S)^2$. Let $w(S)$ denote the least-squares solution
of the subproblem $(X_S,y_S)$ (the interpolant when $X_S$ is invertible).
\citet{DW17} proved two results. The first is unbiasedness:
$\E[w(S)]=w^\star$ for every label vector, requiring only full row rank
(their Theorem~3 and Proposition~7). The second bounds the total square loss $L$ on
the full data:
\begin{equation}\label{eq:dw}
  \E\bigl[L(w(S))\bigr]\ \le\ (d+1)\,L(w^\star),
\end{equation}
with \emph{equality} when $X$ is in general position (every $d$-column
submatrix nonsingular), as stated in their Theorem~5. We only use the inequality
\eqref{eq:dw}, which holds without general position. Equality in
\eqref{eq:dw} can fail for degenerate $X$ (e.g.\ $x_1=x_2=e_1$,
$x_3=x_4=e_2$, $y=(1,-1,1,-1)$ has $\E[L(w(S))]=2L(w^\star)$). Dividing by
$N$ converts \eqref{eq:dw} to average losses.

\begin{proof}[Proof of Theorem~\ref{thm:lowbudget}]
Only $\Fw(d,m,d)\le d+1$ remains. Fix a dataset and let $r=\rank$ of the
features. If $r=0$, the ratio is $1$ because both rules return $W=0$. If
$\Lstar=0$, choose $r$ points whose features form a basis of $T$ with
positive weights. The weighted problem then has minimum $0$ and minimizer set
$W^\circ+\{H:HP_T=0\}$. The Frobenius rule returns $W^\circ$, so the ratio
is $1$ by the $0/0$ convention.

Otherwise, choose an orthonormal basis $U\in\R^{d\times r}$ of $T$ and write
$x_i=U\xi_i$ with $\Xi=[\xi_1\cdots\xi_N]\in\R^{r\times N}$ of full row
rank. Run size-$r$ volume sampling on $\Xi$. Crucially, \emph{the
distribution over subsets depends only on the features}. For the $q$-th
output row, \eqref{eq:dw} (in dimension $r$) gives
$\E\,L^{(q)}_D(w_{q,S})\le(r+1)L^{(q),\star}_D$. Summing over the $m$ rows
and using Lemma~\ref{lem:rows},
\[
  \E\bigl[L_D(W_S)\bigr]\ \le\ (r+1)\,\Lstar\ \le\ (d+1)\,\Lstar ,
\]
where $W_S=\Theta_SU^\top$ is the matrix whose rows are the per-row
subproblem solutions in $\Xi$-coordinates with the $T^\perp$-component set to
zero. Hence some subset $S$ (of $r\le d$ points) achieves
$L_D(W_S)\le(d+1)\Lstar$. This $W_S$ is realizable by a legal selection: the
$r$ chosen features are linearly independent in $T$. Thus, with any positive
weights, the weighted ERM set consists of all interpolants of the $r$ selected
points. The Frobenius rule returns exactly $W_S$ (the interpolant with
vanishing $T^\perp$-component). Padding to $d$ slots by
Lemma~\ref{lem:support} completes the proof.
\end{proof}

%% file: app-near.tex
\section{Proofs for Sections~\ref{sec:rigidity} and~\ref{sec:nearthreshold}}\label{app:near}

Throughout this section, fix a dataset with residual dyads
$M_i=\rho_ix_i^\top$ and recall the minimal certificate size $\tau(D)$ from
\eqref{eq:tau}. Theorem~\ref{thm:upper} gives $\tau(D)\le\max\{1,(m+1)r\}$.
The upper bound at budget $(m+1)d-1$ requires understanding the datasets with
$\tau(D)=(m+1)d$, which we call \emph{certificate-maximal}. The key is a
rigidity theorem for maximal certificates.

\subsection{Rigidity of maximal certificates}\label{loc:proof-rigidity}

\begin{restate}{Theorem~\ref{thm:rigidity} (Determinant--facet rigidity; restated)}
Let the features span $\R^d$ and suppose $\tau(D)=k:=(m+1)d$. Let $c$ be a
spanning zero certificate with $|\supp c|=k$ and support $S$. Then $S$ splits
uniquely (up to permutation) into $d$ pairwise disjoint blocks
$S=C_1\sqcup\dots\sqcup C_d$ with $|C_j|=m+1$, where each $C_j$ supports a
positive circuit of the dyads whose features span a one-dimensional line
$U_j$, and $U_1,\dots,U_d$ are linearly independent.
\end{restate}

\begin{proof}
The proof has four steps. We first identify the normalized nonnegative kernel as a
polytope, then use the determinant polynomial to bound its number of facets.
This shows that the polytope is a simplex. Finally, we read off line-supported
circuits from the extreme rays. Vectorize the dyads into
$\mathsf A=[\mathrm{vec}(M_i)]_{i\in S}\in\R^{md\times k}$ and let
$K=\ker\mathsf A\cap\R^k_{\ge0}$, $h=\dim\ker\mathsf A\ge k-md=d$. Since
$c>0$ on $S$ and lies in $\ker\mathsf A$, a neighborhood of $c$ in
$\ker\mathsf A$ remains nonnegative. Thus $K$ spans $\ker\mathsf A$, and the
compact section $P=K\cap\{\ones^\top\theta=1\}$ is a polytope of dimension
$h-1$ whose relative interior consists exactly of the strictly positive
points.

Consider $f(\theta)=\det\bigl(\sum_{i\in S}\theta_ix_ix_i^\top\bigr)$, a polynomial of
degree at most $d$. On $\relint P$, all weights are positive and the features
span $\R^d$, so $f>0$. At any boundary point, some coordinate vanishes. If
$f>0$ there, the positively weighted features would span $\R^d$ and the point
would be a spanning zero certificate of support $\le k-1$, contradicting
$\tau=k$. Hence $f\equiv0$ on every facet of $P$. Parametrizing
$\mathrm{aff}\,P$ by $\R^{h-1}$, the restriction $g$ of $f$ is a nonzero
polynomial of degree $\le d$. It vanishes on the relative interior of each
facet and is therefore divisible by the pairwise non-associated affine linear forms
of the distinct facet hyperplanes. Thus $P$ has at most $d$ facets.
If $h\ge2$, then $P$ has positive dimension and, being bounded, at least $h$
facets, so $h\le d$. If $h=1$, then $h\le d$ trivially. Together with
$h\ge d$, this gives $h=d$. Thus, for $d\ge2$, $P$ is a $(d-1)$-dimensional
polytope with exactly $d$ facets and hence a $(d-1)$-simplex. For $d=1$,
$P$ is a single point. In both cases, $K$ is a simplicial cone with extreme
rays $v^{(1)},\dots,v^{(d)}$.

Let $U_j=\spn\{x_i:v^{(j)}_i>0\}$. As $c$ is a strictly positive combination
of all rays, the supports of the rays cover $S$, so $\sum_jU_j=\R^d$. If
$d=1$, there is a single ray and $U_1=\R$, so $\dim U_1=1$. Now let $d\ge2$.
Dropping ray $j$ and normalizing the sum of the remaining rays gives a
boundary point of $P$, where the features cannot span, so
$\sum_{\ell\ne j}U_\ell\ne\R^d$ for each $j$: the family $\{U_j\}$ is a
minimal spanning family. Choose functionals $\phi_j$ vanishing on
$\sum_{\ell\ne j}U_\ell$ and not on $U_j$, and normalize them against vectors
$u_j\in U_j$ with $\phi_j(u_j)=1$. Then $(\phi_j)$ is a dual basis, and
$U_j\subseteq\bigcap_{i\ne j}\ker\phi_i$, a one-dimensional space, so
$\dim U_j=1$ in every dimension $d\ge1$.

Each extreme ray of the cone $\{\theta\ge0:\mathsf A\theta=0\}$ has
inclusion-minimal support $C_j$, and $\dim\ker\mathsf A_{C_j}=1$. Otherwise, a two-sided perturbation along a
second kernel direction would split the ray. All dyads of $C_j$ lie in
$\R^m\otimes U_j$, a space of dimension $m$, so
$|C_j|=\rank\mathsf A_{C_j}+1\le m+1$. Finally,
$k=|S|\le\sum_j|C_j|\le d(m+1)=k$ forces every $|C_j|=m+1$ and pairwise
disjoint supports. The decomposition is unique because the extreme rays of
$K$ are.
\end{proof}

\subsection{Certificate-maximal datasets live on \texorpdfstring{$d$}{d} lines}\label{loc:proof-lines}

Theorem~\ref{thm:rigidity} describes one maximal certificate. We now use the
maximality of $\tau(D)$ to show that the same line-circuit structure governs
every nonzero-feature point of the dataset.

\begin{restate}{Theorem~\ref{thm:lines} (Line structure; restated)}
Let the features span $\R^d$ and $\tau(D)=(m+1)d=k$. Then:
\begin{enumerate}
\item every data point with $x_i\ne0$ belongs to some positive circuit of
  dyads supported on a single feature line and of size exactly $m+1$;
\item there is no positive circuit of dyads whose features span a subspace of
  dimension $\ge2$, and no line circuit of size $\le m$; in particular no
  point has $x_i\ne0$ and $\rho_i=0$;
\item the nonzero features lie on exactly $d$ linearly independent lines.
\end{enumerate}
Consequently, any dataset whose nonzero features are \emph{not} of this
$d$-line form satisfies $\tau(D)\le(m+1)d-1$ and is recovered exactly within
budget $(m+1)d-1$.
\end{restate}

\begin{proof}
(1) Extend $x_t$, with $t$ fixed and $x_t\ne0$, to a feature basis $B\ni t$ and apply
Lemma~\ref{lem:compression}: the resulting certificate has support at most
$k$ and, since $\tau=k$, exactly $k$. Theorem~\ref{thm:rigidity} splits it
into $(m+1)$-point line circuits, one of which contains $t$.

(2) Let $C$ be a positive circuit whose features span an $s$-dimensional
space with $s\ge2$. Then $|C|\le ms+1$ because its dyads lie in an $ms$-dimensional
space. Extend the features of $C$ by $d-s$ data points to a basis of
$\R^d$, and attach to each added point its $(m+1)$-point line circuit from
(1). Summing all these positive relations gives a spanning zero certificate
of support at most $ms+1+(d-s)(m+1)=k-s+1\le k-1$, contradicting $\tau=k$.
A line circuit of size $\le m$ similarly completes with $d-1$ line circuits
to support $\le m+(d-1)(m+1)=k-1$. A point with $x_i\ne0=\rho_i$ is by
itself a $1$-point line ``circuit'' ($M_i=0$) and is excluded the same way.

(3) Suppose $d+1$ distinct nonzero feature lines exist. Take an
$(m+1)$-point line circuit on each using (1), and let $J$ be their
disjoint union, $|J|=(d+1)(m+1)$. Then
$\dim\ker\mathsf A_J\ge|J|-md=d+m+1$. The sum of the circuits' positive kernel
vectors is strictly positive on $J$, so the nonnegative kernel cone spans
$\ker\mathsf A_J$. But each extreme ray of that cone is a positive circuit, hence by
(2) supported on a single line and of size $m+1$. On each line, the kernel of
the corresponding block is one-dimensional, so there are at most $d+1$
extreme rays, spanning at most $d+1$ dimensions. This contradicts
$d+m+1>d+1$. As the features span $\R^d$, there are exactly $d$ independent
lines.
\end{proof}

\subsection{The \texorpdfstring{$m$}{m}-point mean lemma}\label{loc:proof-mean}

For a positive $(m+1)$-point circuit, deleting one vertex leaves an
$m$-point mean. The lemma shows that some deletion moves the mean by at most
a $1/m$ fraction of the root-mean-square radius.

\begin{restate}{Lemma~\ref{lem:mmean} ($m$-point mean lemma; restated)}
Let $z_i\in\R^m$ carry weights $p_i>0$ with $\sum_ip_i=1$ and
$\sum_ip_iz_i=0$, and let
$E=\sum_ip_i\Wnorm{z_i}^2$. Then some convex combination $\delta$ of at most
$m$ of the points satisfies $\Wnorm\delta^2\le E/m^2$. The constant is
attained by the uniformly weighted regular simplex.
\end{restate}

\begin{proof}
By conic Carath\'eodory, decompose $p$ inside the cone
$\{q\ge0:\sum_iq_iz_i=0\}$ as a positive combination of positive circuits
(extreme rays), each of support at most $m+1$, and normalize each circuit to
a probability vector $q^{(\alpha)}$. Thus $p=\sum_\alpha\lambda_\alpha q^{(\alpha)}$
with $\lambda_\alpha>0$ and $\sum_\alpha\lambda_\alpha=\sum_ip_i=1$. The
second moments $E_{q^{(\alpha)}}$ then average to $E$, so some circuit $q$
has $E_q\le E$. If $E_q=0$, all its points vanish and a single point gives
$\delta=0$. If its support has size $\le m$, its own zero mean is the desired
$\delta=0$. Otherwise, the support has weights $q_1,\dots,q_{m+1}>0$ with
$\sum_jq_jz_j=0$. Deleting point $j$ leaves the convex combination
$\mu_j=-\frac{q_j}{1-q_j}z_j$ of the other $m$ points. With
$t_j=q_j\Wnorm{z_j}^2/E_q$ (so $\sum_jt_j=1$),
$\Wnorm{\mu_j}^2/E_q=q_jt_j/(1-q_j)^2$. The Cauchy--Schwarz inequality gives
$\sum_j\frac{(1-q_j)^2}{q_j}=\sum_j\frac1{q_j}-2(m+1)+1\ge(m+1)^2-2(m+1)+1
=m^2$. Hence not all $j$ can violate
$\Wnorm{\mu_j}^2\le E_q/m^2$, for otherwise summing
$t_j>\frac{(1-q_j)^2}{m^2q_j}$ over $j$ yields $1>1$. Finally,
$E_q\le E$. For tightness, consider the uniformly weighted regular simplex with
vertices $r_1,\dots,r_{m+1}$, so $E=\Wnorm{r_j}^2$ for every $j$. Every set
of at most $m$ vertices lies in a facet. The point of the facet opposite
$r_j$ closest to the origin is its centroid $-r_j/m$, and
$\Wnorm{r_j/m}^2=E/m^2$. Hence no convex combination of at most $m$ vertices
has smaller squared norm, and the deletions $\mu_j$ attain it.
\end{proof}

\input{appendix-visual-mean}

\subsection{Assembly}\label{loc:proof-near}

The budget deficit is exactly one point. We therefore fit all but one feature
line exactly and spend only $m$ points on the line with the smallest residual
energy.

\begin{restate}{Theorem~\ref{thm:nearthreshold} (Near-threshold value; restated)}
For all $d,m\ge2$ (and in fact for all $d,m\ge1$, cf.\
Remark~\ref{rem:range}),
$\Fw\bigl(d,m,(m+1)d-1\bigr)=1+\dfrac1{dm^2}$.
\end{restate}

\begin{proof}
\emph{Upper bound.} Let $n_0=(m+1)d-1$ and fix a dataset. If the feature
rank is $r<d$, Theorem~\ref{thm:upper} gives a certificate of support at
most $\max\{1,(m+1)r\}\le n_0$, with one point when $r=0$, and the ratio is $1$. The same holds if $\Lstar=0$, by basis
selection as in Section~\ref{sec:lowbudget}, or if $\tau(D)\le n_0$. In the
remaining certificate-maximal case $\tau(D)=(m+1)d$, Theorem~\ref{thm:lines} provides a
basis $u_1,\dots,u_d$ and scalars $s_i\ne0$ with $x_i=s_iu_{j(i)}$ for every
nonzero-feature point. Reading the first-order condition
$\sum_iM_i=0$ against the basis $(u_j)$, each line satisfies
$\sum_{i\in I_j}s_i\rho_i=0$. Set, per line,
\[
  z_i=\rho_i/s_i,\qquad p_i=s_i^2/g_j,\qquad g_j=\sum_{i\in I_j}s_i^2,
  \qquad R_j=\sum_{i\in I_j}\Wnorm{\rho_i}^2 ,
\]
so that $\sum_{i\in I_j}p_iz_i=0$ and $\sum_ip_i\Wnorm{z_i}^2=R_j/g_j$.
Choose the lightest line $j_\star$ with
$R_{j_\star}\le\frac1d\sum_jR_j\le\frac{N\Lstar}d$ (features equal to $0$
only increase $\Lstar$). On each line $j\ne j_\star$, the probability vector $(p_i)_{i\in I_j}$ is a
nonzero element of the pointed cone $\{\theta\ge0:\sum_{i\in I_j}\theta_iz_i=0\}$.
Decompose it into extreme rays, take one of them, and normalize it to
coordinate sum one. This gives a positive zero-mean combination of the $z_i$
supported on a positive circuit, hence on at most $m+1$ points
(Definition~\ref{def:circuit}). On line $j_\star$ pick at most $m$ points and a convex
combination $\delta$ with
$\Wnorm\delta^2\le\frac{R_{j_\star}}{m^2g_{j_\star}}$
(Lemma~\ref{lem:mmean}). Translate the per-line convex coefficients
$\alpha_i$ into selection weights $c_i=\alpha_i/s_i^2>0$, rescaling each line
independently and normalizing globally. The weighted normal equations then
decouple along the basis and give $\widehat W=W^\circ+H$ with $Hu_j=0$ for
$j\ne j_\star$ and $Hu_{j_\star}=-\delta$. The selected features contain all
$d$ lines, so the weighted Gram matrix is positive definite and the
minimizer is unique. The support has size at most $(d-1)(m+1)+m=n_0$.
Pad the selection by Lemma~\ref{lem:support} if fewer slots are used. By
\eqref{eq:lossexp},
\[
  L_D(\widehat W)-\Lstar=\frac{g_{j_\star}\Wnorm\delta^2}{N}
  \le\frac{R_{j_\star}}{Nm^2}\le\frac{\Lstar}{dm^2}.
\]

\emph{Lower bound.} Take the regular-simplex axial instance: let
$r_1,\dots,r_{m+1}\in\R^m$ be regular simplex vertices with
$\sum_jr_j=0$ and $\Wnorm{r_j}^2=\frac m{m+1}$, whose facets are at squared
distance $h^2=\frac1{m(m+1)}$ from the origin. Put $w=he_1$ and
$D=\{(e_i,\,w+r_j)\}_{i\in[d],j\in[m+1]}$. The unique full minimizer is
$W^\circ=[w\ \cdots\ w]$ with $\Lstar=\frac m{m+1}$. For any selection of
support at most $N-1$, some axis is either empty, in which case the Frobenius
rule sets its column to zero with deviation $\Wnorm w^2=h^2$, or represented by a proper subset, with deviation at
least the facet distance $h^2$. By the analogue of
\eqref{eq:axialloss}, the loss exceeds $\Lstar$ by at least $h^2/d$, and
deleting one vertex with uniform facet weights attains it. The ratio is
$1+\frac{h^2/d}{m/(m+1)}=1+\frac1{dm^2}$.
\end{proof}

\begin{remark}[General centered-simplex families]\label{rem:family}
The same computation applies to any centered simplex
$r_1,\dots,r_{m+1}$ ($\sum_jr_j=0$, $S=\sum_j\Wnorm{r_j}^2$,
$h=\min_j\dist(0,\conv\{r_k\}_{k\ne j})$), provided the columns of
$W^\circ$ are chosen with $\Wnorm{W^\circ e_i}\ge h$: the exact
$(N-1)$-point ratio is $1+\frac{(m+1)h^2}{dS}$. Within this family,
$h^2\le S/(m^2(m+1))$ with equality exactly for regular simplices. Indeed,
the centroid of the facet opposite $r_j$ is $-r_j/m$, so
$h^2\le\min_j\Wnorm{r_j}^2/m^2\le S/(m^2(m+1))$. Equality forces all
$\Wnorm{r_j}$ to be equal and each facet centroid to be the point of that
facet nearest the origin, i.e.\ $\langle r_j,r_k\rangle$ independent of
$k\ne j$. Thus, in the equality case, all pairwise inner products coincide and the simplex is
regular. For $m\ge2$, the integer instance of Definition~\ref{def:axial}
realizes the slightly smaller rational value of Theorem~\ref{thm:axial} with
integer data. For $m=1$, the two values coincide.
\end{remark}

\input{appendix-visual-near}

%% file: appendix-visual-mean.tex
\begin{figure}[H]
\centering
\begin{tikzpicture}[x=1pt,y=1pt,font=\normalsize,
  every node/.style={inner sep=2pt},line width=0.6pt]
  \coordinate (r1) at (72,112);
  \coordinate (r2) at ({72-30*sqrt(3)},22);
  \coordinate (r3) at ({72+30*sqrt(3)},22);
  \coordinate (origin) at (72,52);
  \coordinate (mu) at (72,22);
  \draw[dashed,black!65] (r2)--(r1)--(r3);
  \draw[line width=1.2pt] (r2)--(r3);
  \draw[->,line width=0.7pt] (origin)--(mu)
    node[midway,right=3pt] {$\|\mu_1\|$};
  \fill (r2) circle (2pt);
  \fill (r3) circle (2pt);
  \draw[fill=white] (r1) circle (2.5pt);
  \fill (origin) circle (1.6pt);
  \fill (mu) circle (1.6pt);
  \node[above=3pt] at (r1) {$r_1$ (removed)};
  \node[below left=2pt] at (r2) {$r_2$};
  \node[below right=2pt] at (r3) {$r_3$};
  \node[left=4pt] at (origin) {$0$};
  \node[below=3pt] at (mu) {$\mu_1$};
  \node[anchor=north west,align=left,text width=223pt] at (164,130) {%
    $m=2$, uniform original weights $p_i=1/3$\par
    $r_1+r_2+r_3=0$,\quad $E=\|r_j\|^2$\par\smallskip
    Retain $r_2,r_3$ with weights $1/2,1/2$:\par
    $\displaystyle \mu_1=\frac{r_2+r_3}{2}=-\frac{r_1}{2}$,\quad
    $\displaystyle \|\mu_1\|^2=\frac E4$.\par\smallskip
    Centered regular $m$-simplex,\\uniformly weighted:\par
    $\displaystyle \mu_j=-\frac{r_j}{m}$,\quad
    $\displaystyle \|\mu_j\|^2=\frac{E}{m^2}$.};
\end{tikzpicture}
\caption{Regular-simplex equality in Lemma~\ref{lem:mmean}.
The diagram lies in the mean-point space. The midpoint of the retained edge
is its nearest point to the origin, giving the squared error $E/4$ for $m=2$.
The opposite-facet centroid gives $E/m^2$ for a centered,
uniformly weighted regular $m$-simplex.
The general upper bound is proved in Appendix~\ref{loc:proof-mean}.}
\label{fig:mean-equality}
\end{figure}
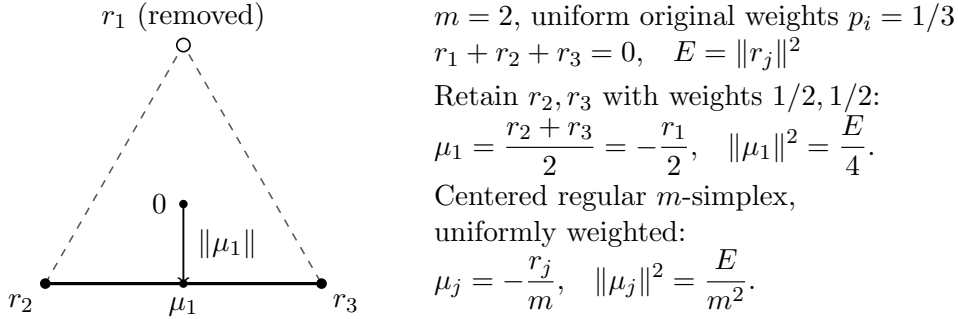

%% file: appendix-visual-near.tex
\newpage
\begin{figure}[H]
\centering
\begin{tikzpicture}[x=1pt,y=1pt,font=\normalsize,
  box/.style={draw=black!65,line width=0.5pt,fill=black!2,
    align=center,inner xsep=5pt,inner ysep=5pt},
  route/.style={->,line width=0.65pt,>=stealth}]
  \node[box,text width=211pt] (S) at (108,0) {%
    \textbf{Certificate-maximal case}\\
    $r=d$, $\Lstar>0$, $\tau(D)=(m+1)d$\\
    All nonzero features lie on\\
    $d$ independent lines\\
    Thm.~\ref{thm:rigidity} + Lem.~\ref{lem:compression}
    $\Rightarrow$ Thm.~\ref{thm:lines}};
  \node[box,text width=150pt] (M) at (307,0) {%
    \textbf{Mean lemma}\quad\ref{lem:mmean}\\
    Zero mean, second moment $E$\\
    At most $m$ points:\\
    $\|\delta\|^2\le E/m^2$};
  \node[box,text width=381pt] (A) at (193,-78) {%
    \textbf{Allocate the one-point deficit}\\
    Choose a line of minimum $R_j$ and apply Lemma~\ref{lem:mmean} there.\\
    Preserve $W^\circ$'s predictions on other lines: $Hu_j=0$ for $j\ne j_\star$.\\
    Sum the excess loss using Lemma~\ref{lem:geometry}.};
  \node[box,text width=180pt] (X) at (92,-151) {%
    \textbf{Easy cases}\\
    $r<d$ or $\Lstar=0$\\
    or $\tau(D)\le n_0$\\
    $R_D(n_0)=1$};
  \node[box,text width=180pt] (U) at (294,-151) {%
    \textbf{Upper bound}\\
    $\displaystyle R_D(n_0)\le1+\frac1{dm^2}$};
  \node[box,text width=180pt] (L) at (92,-218) {%
    \textbf{Independent lower bound}\\
    Regular-simplex axial instance\\
    Empty axis or nonempty\\
    proper subset};
  \node[box,text width=180pt] (T) at (294,-218) {%
    \textbf{Theorem~\ref{thm:nearthreshold}}\\
    $\displaystyle \Fw(d,m,n_0)=1+\frac1{dm^2}$};
  \draw[route] (S.south)--(S.south |- A.north);
  \draw[route] (M.south)--(M.south |- A.north);
  \draw[route] (A.south -| U.north)--(U.north);
  \draw[route] (X.east)--(U.west);
  \draw[route] (U.south)--(T.north);
  \draw[route] (L.east)--(T.west);
\end{tikzpicture}
\par\smallskip
\(
 |\supp c|\le(d-1)(m+1)+m=n_0,
 \qquad n_0=(m+1)d-1.
\)\par\smallskip
\(
 L_D(\widehat W)-\Lstar
 \le \dfrac{R_{j_\star}}{Nm^2}
 \le \dfrac{\Lstar}{dm^2}.
\)
\caption{Local proof assembly for Theorem~\ref{thm:nearthreshold}.
The upper bound combines the easy cases with the certificate-maximal case.
The mean lemma is an independent input. The regular-simplex axial construction
supplies the matching lower bound. The arrows summarize the assembly in
Appendix~\ref{loc:proof-near}. The structural ingredients are proved in
Appendices~\ref{loc:proof-rigidity}--\ref{loc:proof-mean}.}
\label{fig:near-assembly}
\end{figure}

%% file: app-cell.tex
\section{The complex dictionary and proofs for Section~\ref{sec:intermediate}}\label{app:cell}

For $d,m\ge2$ and $d<n<(m+1)d-1$, the profile $\Fw(d,m,n)$ remains open. For
$d=1$, it is given by Corollary~\ref{cor:d1}. The corresponding scalar
regime also remains open, apart from the cells settled in \citep{ZhangScalar}:
it is the intermediate regime of Question~2 of \citep{HMSYnote}.
In this section we treat the smallest cell $(d,m)=(2,2)$, where the open
budgets are $n=3,4$. We prove the interval theorem
(Theorem~\ref{thm:intervals}) and reduce the conjectured exact values to a
finite moment problem. The evidence for the conjecture is collected in
Appendix~\ref{sec:evidence}.

\subsection{Lower bounds: an explicit axial instance}\label{loc:proof-lower22}

\begin{restate}{Proposition~\ref{prop:lower22} (restated)}
There is an explicit dataset $D^\triangle$ with
\[
  \frac{L^{\star}_{D^\triangle}(3;\mathrm{weighted})}
       {L^{\star}_{D^\triangle}}=\frac{13}8,
  \qquad
  \frac{L^{\star}_{D^\triangle}(4;\mathrm{weighted})}
       {L^{\star}_{D^\triangle}}=\frac54 .
\]
Hence $\Fw(2,2,3)\ge\frac{13}8$ and $\Fw(2,2,4)\ge\frac54$.
\end{restate}

\begin{proof}
Let $r_1,r_2,r_3\in\R^2$ be the vertices of a centered equilateral triangle
with $\Wnorm{r_j}^2=\frac23$. For instance, take $r_1=\sqrt{\tfrac23}\,(1,0)$ and
$r_{2,3}=\sqrt{\tfrac23}\,(-\tfrac12,\pm\tfrac{\sqrt3}2)$. The total is then
$S=2$ and the squared distance from the origin to an edge midpoint is
$\frac16$. Let $w=\sqrt{\tfrac56}\,(1,0)$, so $\Wnorm{w}^2=\frac56$. Take
$D^\triangle=\{(e_i,\,w+r_j)\}_{i\in[2],j\in[3]}$. As in
Lemma~\ref{lem:axialbasic}, $W^\circ=[w\ w]$ and $\Lstar=\frac23$. The same
calculation shows that a selection allocating $k_i$ points to axis $i$ incurs excess
$\frac12(c(k_1)+c(k_2))$ where $c(3)=0$, $c(2)=\frac16$ (best edge),
$c(1)=\frac23$ (best vertex), $c(0)=\Wnorm w^2=\frac56$ (empty axis). For
budget $3$ the minimum over allocations is
$c(2)+c(1)=c(3)+c(0)=\frac56$, giving ratio
$1+\frac{5/12}{2/3}=\frac{13}8$. For budget $4$, the minimum is
$2c(2)=\frac13$, giving $1+\frac{1/6}{2/3}=\frac54$.
\end{proof}

\subsection{A complex dictionary for the hard branches}
\label{subsec:dictionary}

We use three kinds of weights below: selection weights $c$ on data points
(Lemma~\ref{lem:support}), full-data atom weights $p$ determined by the
features, and selected atom weights $q$. The correspondence between $c$ and
$q$ is given after the definitions.

Fix a $(2,2)$ dataset with feature rank $2$ and $\Lstar>0$. We only consider
selections whose features span $\R^2$. For these selections, the weighted
minimizer is unique, so Frobenius tie-breaking plays no role. Their predictions
and full-data loss are invariant under an invertible linear change of features
$x\mapsto Gx$, with $W\mapsto WG^{-1}$. For a rank-deficient selected objective,
such a change need not preserve the minimum-Frobenius-norm tie-breaking.
This is why we restrict to spanning selections. Since our aim here is to
exhibit good spanning selections, we may assume that the data are whitened:
$\sum_ix_ix_i^\top=\gamma I_2$. Identify $\R^2\cong\C$, and write $\xi_i$ for the
feature and $\eta_i$ for the residual $\rho_i$ of point $i$, viewed in
$\C$. Let $I=\{i:\xi_i\ne0\}$, $S=\sum_{j\in I}|\xi_j|^2$,
$R_0=\sum_{i\notin I}|\eta_i|^2$, and for $i\in I$ set
\[
  p_i=\frac{|\xi_i|^2}{S},\qquad
  t_i=\frac{\bar\xi_i}{\xi_i}\in S^1,\qquad
  z_i=\frac{\eta_i}{\xi_i}\in\C,\qquad
  E=\sum_{i\in I}p_i|z_i|^2 .
\]
Whitening is equivalent to $\sum_i\xi_i^2=0$, that is, $\E_p t=0$. The
first-order condition $\sum_i\rho_ix_i^\top=0$ is equivalent to the pair
$\E_p z=0$, $\E_p(\bar tz)=0$. A selection with weights $c$ corresponds to
$q_i\propto c_i|\xi_i|^2$ on $I$; zero-feature points do not influence the
selected minimizer. Writing $\mu=\E_qt$, $\nu=\E_qz$,
$\lambda=\E_q(\bar tz)$, represent $W^\circ-\widehat W$ as the real-linear
map $\xi\mapsto a\xi+b\bar\xi$, so that the residual of point $i$ under
$\widehat W$ is $\eta_i-(a\xi_i+b\bar\xi_i)=\xi_i(z_i-a-bt_i)$. The weighted
normal equations then read
\begin{equation}\label{eq:ab}
  \begin{pmatrix}1&\mu\\ \bar\mu&1\end{pmatrix}
  \begin{pmatrix}a\\b\end{pmatrix}
  =\begin{pmatrix}\nu\\\lambda\end{pmatrix},
  \qquad
  |\mu|<1\iff\text{selected features span }\R^2 ,
\end{equation}
and, with $C(q)=|a|^2+|b|^2$, the excess loss satisfies the following bound,
using $\E_pt=0$:
\[
  \frac{L_D(\widehat W)-\Lstar}{\Lstar}
  =\frac{S\,C(q)}{S\,E+R_0}\ \le\ \frac{C(q)}{E}\qquad(E>0),
\]
with equality when there are no zero-feature points. If $E=0$, every
residual dyad vanishes and any feature basis with positive weights is a
spanning zero certificate, so the ratio is $1$. Thus $E>0$ in every
nonexact branch. Conversely, every finite system $(p_i,t_i,z_i)$ with the
three moment conditions is realized by a whitened dataset without
zero-feature points: take $\xi_i=\sqrt{p_i}s_i$ with $\bar s_i/s_i=t_i$ and
$\eta_i=z_i\xi_i$. For this realization, the excess ratio of a spanning
selection is exactly $C(q)/E$.

For $k\in\{3,4\}$ define
\[
  \Gamma_k=\sup_{\substack{\text{moment systems}\\ E>0}}\ \frac1E\,
  \inf_{\substack{|\supp q|\le k\\ |\mu(q)|<1}}C(q).
\]

\subsection{Interpolation identities}\label{loc:proof-interp}

\begin{restate}{Lemma~\ref{lem:interp} (Interpolation convex combination; restated)}
For $t_i\ne t_j$ let
$B_{ij}=\Bigl(\frac{t_iz_j-t_jz_i}{t_i-t_j},\ \frac{z_i-z_j}{t_i-t_j}\Bigr)
\in\C^2$ (the coefficients of the affine function of $t$ through the two
points) and $d_{ij}=|t_i-t_j|^2$. Then for every admissible $q$,
\[
  \Delta(q):=1-|\mu|^2=\sum_{i<j}q_iq_jd_{ij},
  \qquad
  (a,b)=\sum_{\substack{i<j\\ t_i\ne t_j}}
  \frac{q_iq_jd_{ij}}{\Delta(q)}\;B_{ij} .
\]
In particular $(a,b)$ is a convex combination of the pairwise interpolation
coefficients; the conditioning factor $(1-|\mu|^2)^{-1}$ is absorbed
entirely.
\end{restate}

\begin{proof}
Expanding $1-|\mu|^2=\frac12\sum_{i,j}q_iq_j|t_i-t_j|^2$ gives the first identity.
With $w_i=\bar t_iz_i$, one checks that
$d_{ij}a_{ij}=(t_i-t_j)(w_i-w_j)$ and
$d_{ij}b_{ij}=(\bar t_i-\bar t_j)(z_i-z_j)$. Summing against $q_iq_j$ and
using $\sum_{i<j}q_iq_j(x_i-x_j)(y_i-y_j)
=\sum_iq_ix_iy_i-(\sum_iq_ix_i)(\sum_iq_iy_i)$ yields
$\sum_{i<j}q_iq_jd_{ij}a_{ij}=\nu-\mu\lambda$ and
$\sum_{i<j}q_iq_jd_{ij}b_{ij}=\lambda-\bar\mu\nu$. The resulting pair equals
$\Delta(q)\cdot(a,b)$ by \eqref{eq:ab}. Same-direction pairs have $d_{ij}=0$ and
are omitted.
\end{proof}

\begin{lemma}[Three-point closure identity]\label{lem:closure}
For $2\le|S|\le3$ let
$K_S=\conv\{B_{ij}:i,j\in S,\ t_i\ne t_j\}$ (discard $S$ with no valid
pair). Then
\[
  \inf_{\substack{\supp q\subseteq S\\ |\mu(q)|<1}}C(q)
  \;=\;\dist^2\bigl(0,K_S\bigr),
  \qquad
  \inf_{\substack{|\supp q|\le3\\ |\mu(q)|<1}}C(q)
  \;=\;\min_{S}\dist^2(0,K_S) .
\]
For three pairwise distinct directions the interior of the edge-weight
triangle is attained by positive weights (via
$q_1{:}q_2{:}q_3=\frac{d_{23}}{\alpha_{23}}{:}\frac{d_{13}}{\alpha_{13}}
{:}\frac{d_{12}}{\alpha_{12}}$), the vertices by two-point supports, and
non-degenerate edge interiors only as limits with $|\mu|\to1$. If two of the
three directions coincide, the corresponding segment of $K_S$ is attained
exactly. Moreover, for three distinct directions the three points
$B_{ij},B_{ik},B_{jk}$, when distinct, are never collinear over $\R$.
\end{lemma}

\begin{proof}
By Lemma~\ref{lem:interp}, for $\supp q\subseteq S$, the value $(a,b)$ ranges
over convex combinations of the valid $B_{ij}$ with weights
$r_{ij}\propto q_iq_jd_{ij}$. The parametrization above shows that all interior
weight profiles are realized. The boundary cases can be checked directly:
for a repeated direction, distributing weight inside the repeated pair moves
along the segment while preserving $|\mu|<1$. Distances to closures equal
infima of continuous functions over the realized sets. To prove non-collinearity, note that
$B_{ik}-B_{ij}$ is a complex multiple of $(-t_i,1)$ and $B_{jk}-B_{ij}$ of
$(-t_j,1)$. Real proportionality of such nonzero vectors forces $t_i=t_j$.
\end{proof}

\begin{remark}[The infimum need not be attained]\label{rem:nonattain}
The inner infimum in $\Fw$ cannot in general be replaced by a minimum, even
for integer data. Let $d=m=2$, $N=5$, with rows $(x_i^\top\mid y_i^\top)$
equal to $(1,0\mid1,4)$, $(0,1\mid-1,1)$, $(1,1\mid-2,0)$, $(1,-1\mid1,1)$,
$(1,2\mid1,4)$. Then $\Lstar=\frac{221}{60}$. By the interpolation identity
of Lemma~\ref{lem:interp} (which uses only the normal equations, not the
whitening), for a spanning selection the deviation $W^\circ-\widehat W$ is a
convex combination of the two-point interpolation deviations $H_{ij}$ with
coefficients proportional to $c_ic_j\det(x_i,x_j)^2$. An exact enumeration
of the ten three-point supports shows that the smallest excess loss over
their closures is $\frac{11}{20N}$, attained only for the support
$\{3,4,5\}$ at the point $\frac{11}{45}H_{34}+\frac{34}{45}H_{45}$ of the
\emph{open edge} joining $H_{34}$ and $H_{45}$. Since the five feature
directions are distinct, a three-point selection with positive weights has
all three edge coefficients positive and a two-point selection is a vertex,
so this point is not realized. Rank-deficient selections are single points
and have larger loss. Hence $L^{\star}_D(3;\mathrm{weighted})=\frac{569}{150}$,
i.e.\ $R_D(3)=\frac{1138}{1105}$, is not attained, although it is approached
by the weights $(0,0,99\varepsilon,36,136\varepsilon)/(36+235\varepsilon)$
as $\varepsilon\downarrow0$: the weight vectors converge to the singleton
$\{4\}$, whose loss is much larger, reflecting the discontinuity of the
minimum-norm rule at a rank drop.
\end{remark}

\begin{lemma}[Global edge identities]\label{lem:edges}
Let $w_{ij}=p_ip_jd_{ij}$. Then
\[
  \sum_{i<j}w_{ij}=1,
  \qquad
  \sum_{\substack{i<j\\ t_i\ne t_j}}w_{ij}B_{ij}=0,
\]
and
\[
  \sum_{\substack{i<j\\ t_i\ne t_j}}w_{ij}\Wnorm{B_{ij}}^2
  \;=\;2E-2\sum_{G}\Bigl(P_G\!\!\sum_{i\in G}p_i|z_i|^2
      -\Bigl|\sum_{i\in G}p_iz_i\Bigr|^2\Bigr)\ \le\ 2E ,
\]
where $G$ ranges over the classes of equal direction and
$P_G=\sum_{i\in G}p_i$. In particular some valid pair has
$\Wnorm{B_{ij}}^2\le2E$.
\end{lemma}

\begin{proof}
The first identity follows from Lemma~\ref{lem:interp} at $q=p$ with $\mu=0$. The
second follows from Lemma~\ref{lem:interp} at $q=p$, where $(a,b)=(0,0)$ by the moment
conditions. For the third, expand
$d_{ij}\Wnorm{B_{ij}}^2=|t_iz_j-t_jz_i|^2+|z_i-z_j|^2$.
When summed over all pairs, the two expressions on the right telescope to
$E-|\E_p(\bar tz)|^2=E$ and
$E-|\E_pz|^2=E$, respectively. In these expanded sums, same-direction pairs,
for which $B_{ij}$ is undefined, contribute $2p_ip_j|z_i-z_j|^2$.
Regrouping these contributions over classes gives the correction term.
\end{proof}

\subsection{Star localization and the two unconditional bounds}

\begin{lemma}[Star localization]\label{lem:star}
Fix an atom $i$ and set $P_{ij}=\frac12p_jd_{ij}$ for $j\ne i$. Then
$\sum_jP_{ij}=1$, the valid star coefficients $B_{ij}$ lie on a real
two-dimensional affine plane through
$\beta_i$ with $\Wnorm{\beta_i}^2=\frac12|z_i|^2$, and
$\sum_jP_{ij}B_{ij}=\beta_i$. Their variance is
$V_i=E+\frac{|z_i|^2}2-L_i$ where
$L_i=\sum_{j:t_j=t_i}p_j|z_j-z_i|^2\ge0$. Consequently there is a convex
combination $\delta$ of at most two valid star coefficients with
\[
  \Wnorm\delta^2\ \le\ \frac E4+\frac{5|z_i|^2}8-\frac{L_i}4 ,
\]
and $\delta$ is a closure point of three-point selections
(Lemma~\ref{lem:closure}). Choosing $i$ with $|z_i|^2\le E$:
\begin{equation}\label{eq:seven8}
  \inf_{\substack{|\supp q|\le3\\ |\mu|<1}}C(q)\ \le\ \frac{7E}8,
  \qquad\text{so}\qquad \Gamma_3\le\frac78 .
\end{equation}
\end{lemma}

\begin{proof}
Put $\beta_i=\bigl(\tfrac{z_i}2,\ \tfrac{\bar t_iz_i}2\bigr)\in\C^2$, so
$\Wnorm{\beta_i}^2=\frac12|z_i|^2$. Every valid star coefficient
$B_{ij}=(a_{ij},b_{ij})$ interpolates atom $i$, so $a_{ij}+t_ib_{ij}=z_i$.
Since $\beta_i$ satisfies the same equation,
$B_{ij}-\beta_i\in\{(-t_i\xi,\xi):\xi\in\C\}$, a real two-dimensional space
orthogonal to $\beta_i$. Indeed,
$\langle(-t_i\xi,\xi),\beta_i\rangle=-t_i\xi\bar z_i/2+\xi t_i\bar z_i/2=0$.
Next, $P_{ij}=\frac12p_j|t_i-t_j|^2=p_j\bigl(1-\Re(\bar t_it_j)\bigr)$ sums
over $j$ to $1-\Re(\bar t_i\E_pt)=1$. Expanding $d_{ij}b_{ij}=(\bar t_i-\bar
t_j)(z_i-z_j)$, $d_{ij}a_{ij}=(t_i-t_j)(\bar t_iz_i-\bar t_jz_j)$ and using
the three moment conditions gives the mean and variance identities
\[
  \sum_jP_{ij}B_{ij}=\beta_i,\qquad
  \sum_jP_{ij}\Wnorm{B_{ij}-\beta_i}^2=E+\frac{|z_i|^2}2-L_i ,
\]
where the sums run over the valid $j$, those with $t_j\ne t_i$. The indices
with $t_j=t_i$ have $P_{ij}=0$ and are the source of the term $L_i$.
After discarding these indices, the remaining $P_{ij}$ form a probability
vector. The centered points $B_{ij}-\beta_i$ have $P$-weighted mean zero in a
real two-dimensional space, so Lemma~\ref{lem:mmean} with $m=2$ produces a
convex combination $v$ of at most two of them with $\Wnorm{v}^2\le V_i/4$.
The same convex combination of the corresponding $B_{ij}$ is
$\delta'=\beta_i+v$, and orthogonality of $\beta_i$ to the centered
directions gives
$\Wnorm{\delta'}^2=\Wnorm{\beta_i}^2+\Wnorm{v}^2
\le\frac{|z_i|^2}2+\frac{V_i}4$, which is the stated bound. The
bound \eqref{eq:seven8} follows from $\min_i|z_i|^2\le E$ and $L_i\ge0$.
\end{proof}

\begin{lemma}[Four-point $E/2$ bound]\label{lem:ehalf}
For every moment system,
$\inf_{|\supp q|\le4,\ |\mu|<1}C(q)\le\frac E2$; hence
$\Gamma_4\le\frac12$.
\end{lemma}

\begin{proof}
Fix $i$ with $|z_i|^2\le E$. By Lemma~\ref{lem:star} the valid star
coefficients $B_{ij}$ lie in a real two-dimensional affine plane and average
to $\beta_i$ under the weights $P_{ij}$. By Carath\'eodory in the plane,
$\beta_i=\sum_{r=1}^{r_0}\alpha_rB_{ij_r}$ for some $r_0\le3$ valid
neighbours and convex weights $\alpha$. Take $q_i=1-\varepsilon$,
$q_{j_r}=\varepsilon c_r$ ($r\le r_0$) with
$c_r\propto\alpha_r/d_{ij_r}$. By Lemma~\ref{lem:interp}, the normalized edge
weights of the star edges tend to $\alpha_r$, while all leaf--leaf edges carry
total weight $O(\varepsilon)$. Thus $(a,b)\to \beta_i$ as
$\varepsilon\downarrow0$, along selections with support $\le4$ and
$|\mu|<1$. Hence the infimum is at most
$\Wnorm{\beta_i}^2=\frac12|z_i|^2\le\frac E2$.
\end{proof}

\subsection{The interval theorem}\label{loc:proof-intervals}

\begin{restate}{Theorem~\ref{thm:intervals} (restated)}
$\Fw(2,2,3)\in\bigl[\frac{13}8,\frac{15}8\bigr]$ and
$\Fw(2,2,4)\in\bigl[\frac54,\frac32\bigr]$.
\end{restate}

\begin{proof}
The lower bounds follow from Proposition~\ref{prop:lower22}. For the upper bounds, fix a dataset
and a budget $n\in\{3,4\}$. If the feature rank is $r\le1$, then
$\tau(D)\le\max\{1,(m+1)r\}\le3\le n$ (Theorem~\ref{thm:upper}; one point
when $r=0$) and the ratio is $1$.
The ratio is also one if $\Lstar=0$ or $\tau(D)\le n$. Otherwise, the dictionary of
Appendix~\ref{subsec:dictionary} applies, and by \eqref{eq:seven8} and
Lemma~\ref{lem:ehalf} (a three-point selection is also a four-point
selection) the ratio is at most $1+\Gamma_3\le\frac{15}8$ for $n=3$ and at
most $1+\Gamma_4\le\frac32$ for $n=4$. The closure points used in
Lemmas~\ref{lem:star}--\ref{lem:ehalf} are limits of admissible selections
within the same budget. This suffices because the inner optimization in
$\Fw$ is an infimum.
\end{proof}

\subsection{The conjecture}\label{subsec:conjecture}

\begin{restate}{Conjecture~\ref{conj:CE} (restated)}
$\Fw(2,2,3)=\frac{13}8$ and $\Fw(2,2,4)=\frac54$. Equivalently, for every
moment system,
\[
  \inf_{|\supp q|\le3,\ |\mu|<1}C(q)\le\frac{5E}8
  \qquad\text{and}\qquad
  \inf_{|\supp q|\le4,\ |\mu|<1}C(q)\le\frac E4 .
\]
\end{restate}

\begin{remark}[The two formulations are equivalent]\label{rem:equiv}
The moment inequalities imply the equalities by the branch analysis in the
proof of Theorem~\ref{thm:intervals} together with the lower bounds of
Proposition~\ref{prop:lower22}. Conversely, suppose a moment system violates
one of the inequalities, say $\inf_{|\supp q|\le k}C(q)=vE$ with $v>c$. Realize it as a
whitened dataset $(x_i,\rho_i)$ as in Appendix~\ref{subsec:dictionary} and
add a large linear signal: $y_i=L\,Ux_i-\rho_i$ with $U$ invertible, so
that $W^\circ=LU$ and the residuals are the given $\rho_i$. For
spanning selections the signal is recovered exactly and the excess ratio is
unchanged, hence at least $v$. For a non-spanning selection with feature span $P_S$,
the minimum-Frobenius-norm solution is $LUP_S+R_SP_S$, whose deviation from
the full minimizer $LU$ has squared Frobenius norm
$L^2\Wnorm{U(I-P_S)}_F^2+\Wnorm{R_SP_S}_F^2\ge L^2\Wnorm{U(I-P_S)}_F^2$
because the two terms act on orthogonal feature subspaces. There are finitely
many non-spanning supports, and $U$ is invertible, so
$\min_S\Wnorm{U(I-P_S)}_F^2>0$ over these supports. For
$L$ large, every non-spanning selection therefore has ratio above $1+v$ as well. Thus
$\Fw(2,2,n)\ge R_D(n)\ge1+v>1+c$ with $c\in\{\tfrac58,\tfrac14\}$,
contradicting the conjectured equalities.
\end{remark}

\subsection{Instance values versus profile values at \texorpdfstring{$(2,2)$}{(2,2)}}

\begin{table}[ht]
\caption{Explicit instances versus global profile values at $(d,m)=(2,2)$.}
\label{tab:instances}
\begin{center}
\begin{tabular}{lcl}
\multicolumn{1}{c}{\bf instance / profile} & \multicolumn{1}{c}{\bf $n$} & \multicolumn{1}{c}{\bf value}\\ \hline \\[-6pt]
$D_{2,2}$ (Corollary~\ref{cor:22}) & $5$ & $43/40$\\
$\Fw(2,2,5)$ (Theorem~\ref{thm:nearthreshold}) & $5$ & $9/8$\\
$D_{2,2}$ (Proposition~\ref{prop:2dfails}) & $4$ & $23/20$\\
$D^\triangle$ (Proposition~\ref{prop:lower22}) & $4$ & $5/4$\\
$\Fw(2,2,4)$ (Theorem~\ref{thm:intervals}) & $4$ & $[5/4,\,3/2]$\\
$D^\triangle$ (Proposition~\ref{prop:lower22}) & $3$ & $13/8$\\
$\Fw(2,2,3)$ (Theorem~\ref{thm:intervals}) & $3$ & $[13/8,\,15/8]$
\end{tabular}
\end{center}
\end{table}

%% file: app-evidence.tex
\section{Evidence for Conjecture~\ref{conj:CE}: full statements and proofs}\label{sec:evidence}

This section records five independent pieces of structural evidence for
Conjecture~\ref{conj:CE}: a finite reduction, the sharp solution of the
two-direction class, a stability statement for its equality fiber, two
obstructions that rule out natural proof strategies, and the classes on
which the conjecture is already proved.

\subsection{Finite reduction: at most seven atoms}\label{loc:proof-seven}

\begin{restate}{Proposition~\ref{prop:seven} (restated)}
For fixed atoms $(t_i,z_i)$, the closed set
$\mathcal P=\{p\in\R^N_{\ge0}:\sum_ip_i=1,\ \E_pt=\E_pz=\E_p(\bar tz)=0\}$
is a polytope cut out by at most $7$ linear equalities, whose points with
full support are the admissible weight vectors. Minimizing the linear
functional $E(p)$ over this polytope yields a basic feasible solution
$p^\star$ of support at most $7$ with $E(p^\star)\le E(p)$. The subsystem on
$\supp p^\star$ is again a moment system, and every selection admissible for
it is admissible for the original system. Consequently each inequality of
Conjecture~\ref{conj:CE} holds
for all finite systems if and only if it holds for all systems with at most
seven atoms. Since atoms may be split into identically placed copies
without changing the problem, it suffices to treat systems with \emph{exactly}
seven labelled atoms.
\end{restate}

\begin{proof}
$E(p)$ is linear in $p$, subject to normalization ($1$ equation) and the
three complex moment constraints ($6$ real equations). A basic optimal solution
has support at most the constraint rank $\le7$. Its subsystem inherits the
moment conditions, and its witness lifts verbatim to the original system.
If the minimum is $E^\star=0$, all its supported $z_i$ vanish. In this case,
$\E t=0$ guarantees two distinct directions, and choosing one atom from
each gives a two-point selection with $C=0$.
Splitting an atom $(p_i,t_i,z_i)$ into
$(\theta p_i,t_i,z_i),((1-\theta)p_i,t_i,z_i)$ changes no moments. Merging
the copies in a witness does not increase its support. The number seven
cannot be lowered by this route: for the seventh-roots system of
Proposition~\ref{prop:mu0} the seven real constraints prescribe the Fourier
coefficients of $p$ at frequencies $0,1,2,3$ and hence, since $p$ is real, at
all frequencies. Thus $p=\frac17\ones$ is the only feasible weight vector.
\end{proof}

\subsection{The two-direction class is solved sharply}

\begin{proposition}\label{prop:twodir}
Suppose the atoms take exactly two directions. Then these directions are
antipodal, and each class has total weight $\frac12$ and zero conditional mean.
Moreover, $C(q)=\frac12(|m_+|^2+|m_-|^2)$, where $m_\pm$ are the conditional means
of the selected weights. Consequently
\[
  \inf_{\substack{|\supp q|\le4\\ |\mu(q)|<1}}C(q)\le\frac E4,\qquad
  \inf_{\substack{|\supp q|\le3\\ |\mu(q)|<1}}C(q)\le\frac{5E}8 ,
\]
and both constants are attained: for six equally weighted atoms with
$z$-values the cube roots of unity over each direction, the three-point
optimum is exactly $\frac{5E}8$ and the four-point optimum exactly
$\frac E4$. Modulo merging atoms with identical $(t,z)$, and for $E>0$,
equality at $\frac{5E}8$ forces each class to be an equilateral $z$-triangle with
conditional weight $\frac13$ at each vertex, and the two classes to have
equal energy. Furthermore, for five-atom two-direction systems (after
merging, class sizes $3{+}2$ or $4{+}1$) the sharp constant at three points
is $\frac47<\frac58$. In the $3{+}2$ case with conditional class energies
$e_3,e_2$,
\[
  \frac1E\inf_{\substack{|\supp q|\le3\\ |\mu(q)|<1}}C(q)
  \ \le\ \frac{\min\{e_3,\ e_3/4+e_2\}}{e_3+e_2}\ \le\ \frac47 ,
\]
with equality at $e_2=\frac34e_3$ for an equilateral triple against a
balanced antipodal pair. The $4{+}1$ case gives at most $\frac14$.
\end{proposition}

\begin{proof}
$\E_pt=0$ with two unit directions forces antipodality and equal class
weights. The two complex moments then force zero conditional class means.
The regression fit at $\pm$ is the pair of selected conditional means,
giving the formula for $C$. Call the smallest squared norm of a convex
combination of two atoms of a class its \emph{two-point cost}, and let $e_G$
be the conditional energy of the class. Lemma~\ref{lem:mmean} (with $m=2$)
bounds the two-point cost by $e_G/4$. For three atoms with conditional
weights $\alpha_i$, this is the deletion bound: deleting one atom leaves the
two-point mean $-\frac{\alpha_iz_i}{1-\alpha_i}$, and
$\sum_i\frac{(1-\alpha_i)^2}{\alpha_i}\ge4$ (Cauchy--Schwarz). A single atom
of a class costs at most $e_G$. Allocate $(2,2)$ or $(2,1)$ across the
classes, both of which must appear in a spanning selection. Balancing the
two allocations yields the bounds. For the six-atom system the two allocations
are tied and give exactly $\frac{5E}8$ and $\frac E4$. For $E>0$, equality
at $\frac{5E}8$ forces equality in the balancing step, i.e.\ equal class
energies, and in each class a two-point cost of exactly $e_G/4$. By the
proof of Lemma~\ref{lem:mmean}, the latter forces every circuit in a
decomposition of the class into positive circuits to have second moment
$e_G$ and two-point cost $e_G/4$. In this equality case, the three-atom
analysis, with equality in Cauchy--Schwarz, therefore forces each circuit
to be a uniformly weighted equilateral triangle of radius $\sqrt{e_G}$.
Two distinct such triangles inscribed in one circle contain a pair of
atoms at angular distance above $120^\circ$, whose segment passes at
squared distance below $e_G/4$ from the origin. Hence,
after merging coincident atoms, each class is a single uniformly weighted
equilateral triangle, as stated. In the five-atom $3{+}2$ case the pair has zero conditional
mean, so two points of it cost $0$ and one point costs at most $e_2$, while
the triple costs at most $e_3/4$ with two points and at most $e_3$ with one.
This gives the displayed bound. The expression $\min\{e_3,e_3/4+e_2\}/(e_3+e_2)$ is
maximized at $e_2=\frac34e_3$ with value $\frac47$, attained as stated. In
the $4{+}1$ case the singleton class has $z=0$, and two points of the
four-atom class cost at most $e_4/4$ by Lemma~\ref{lem:mmean}, giving at
most $\frac14$.
\end{proof}

\subsection{Local stability of the reduced equality fiber}

\begin{proposition}\label{prop:stability}
Let $\mathcal E$ be the set of reduced $3{+}3$ two-direction equality systems
of Proposition~\ref{prop:twodir} (two uniformly weighted equilateral
$z$-triangles of equal radius on antipodal directions). Then:
(i) for a system consisting of two uniformly weighted equilateral
$z$-triangles on antipodal directions with class energies $E(1\pm\eta)$,
the normalized three-point value $E^{-1}\inf_{|\supp q|\le3,\,|\mu|<1}C$
equals $\frac18(5-3|\eta|)$, so within this stratum the equality fiber
($\eta=0$, value $\frac58$) sits on a downward cusp;
(ii) for every sequence of moment systems with a fixed number of labelled
atoms converging to a point of $\mathcal E$ in which some direction class
genuinely splits (two atoms of one limiting direction class have distinct
directions for all sufficiently large indices),
\[
  \limsup\ \frac1E\inf_{|\supp q|\le3,\,|\mu|<1}C\ \le\ \frac12\ <\
  \frac58 .
\]
Hence no nearby ascent direction exists: $\mathcal E$ is a local maximum of
the normalized three-point value $E^{-1}\inf C$ in this stratified sense.
(Only the normalized value is meaningful here: rescaling all $z_i$ by a
common factor multiplies $C$ and $E$ by the same constant.)
\end{proposition}

\begin{proof}
(i) A spanning three-point selection takes two atoms from one class and one
from the other. In a uniformly weighted equilateral triangle of energy $e$
the best two-point mean has squared norm $e/4$ (an edge midpoint) and the
best single point has squared norm $e$. By the formula for $C$, the value
is $\frac12\min\{e_+/4+e_-,\,e_++e_-/4\}=\frac E8(5-3|\eta|)$ for
$e_\pm=E(1\pm\eta)$. Dividing by $E$ gives the normalized value
$\frac18(5-3|\eta|)$. For (ii),
let atoms $j,k$ of the $+$ class split: $\delta=t_j-t_k\to0$, $\delta\ne0$,
with $t_j,t_k\to t_\ast$ and $z_j,z_k$ tending to distinct points of the limit
equilateral triangle (radius $A$, $E\to A^2$). Such a pair is available at
every large index: if two atoms of distinct directions converge to the same
vertex, at least one of them differs in direction from any fixed atom
converging to another vertex of the class, and that pair may be used
instead. Only finitely many pairs occur, so the estimates below are uniform
in the choice. Use the prediction
coordinates $U_\pm(a,b)=a\pm t_\ast b$, in which
$C=\frac12(|U_+|^2+|U_-|^2)$ at the limit directions. Writing
$t_j=t_\ast+\eta_j$, $t_k=t_\ast+\eta_k$, the same-class edge $V=B_{jk}$
satisfies the exact identities
\[
  U_+(V)=\frac{\eta_jz_k-\eta_kz_j}{\delta},\qquad
  U_-(V)=\frac{(2t_\ast+\eta_j)z_k-(2t_\ast+\eta_k)z_j}{\delta},
\]
so that $|\delta|\,U_+(V)\to0$ while
$R_\delta:=|\delta|\,U_-(V)$ has $|R_\delta|\to2|z_j-z_k|>0$ (its phase need
not converge). Pick an atom $\ell$ of the $-$ class. The two cross edges
$B_{j\ell},B_{k\ell}$ average, in the $U$-coordinates, to
$\bigl(\frac{z_j+z_k}2,\ z_\ell\bigr)+o(1)$, whose first coordinate has
modulus $\frac A2$. Among the three equilateral vertices $z_\ell$ one always
satisfies $\Re(\bar z_\ell R_\delta)\le-\frac12A|R_\delta|$ (some vertex
makes an angle $\le\pi/3$ with $-R_\delta$). Give the same-class edge the
convex weight $s_\delta=c_\delta|\delta|$ with
$c_\delta=-\Re(\bar z_\ell R_\delta)/|R_\delta|^2=O(1)$, and the two cross
edges weights $(1-s_\delta)/2$ each. This is a valid point of the triangle
closure $K_{\{j,k,\ell\}}$ (Lemma~\ref{lem:closure}). In the limit the
first coordinate is unchanged ($s_\delta U_+(V)\to0$) and the second is
optimized along the ray: $|z_\ell+c_\delta R_\delta|^2\le\frac34A^2+o(1)$.
Hence $\limsup\inf C_3\le\frac12\bigl(\frac{A^2}4+\frac{3A^2}4\bigr)
=\frac{A^2}2$. If neither class splits, the system stays in the
two-direction stratum, where the value is at most $\frac{5E}8$ with the
rigid equality characterization of Proposition~\ref{prop:twodir}.
\end{proof}

\subsection{Two obstructions: proof templates that cannot work}

\begin{proposition}[$\mu=0$ obstruction]\label{prop:mu0}
A proof of Conjecture~\ref{conj:CE} cannot restrict to selections with
$\mu=0$. Concretely, for the seventh-roots system
$t_j=\zeta^j,\ z_j=\zeta^{3j},\ p_j=\frac17$ ($\zeta=e^{2\pi i/7}$), every
$q$ with $\mu=0$ satisfies $C\ge\frac{7/k-1}2$ for support $k$ (a discrete
Parseval identity). For these constrained selections, $C\ge\frac38>\frac14$ at $k=4$ and
$C\ge\frac23>\frac58$ at $k=3$. Yet unconstrained selections achieve
$C=\frac27$ at three points (the difference set $\{0,1,3\}$) and $C<\frac14$
at four consecutive points (an exact algebraic computation modulo
$x^3+x^2-2x-1$). Similarly, for the regular hexagon third-harmonic system
($t_j=e^{\pi ij/3}$, $z_j=(-1)^j$, $p_j=\frac16$) all $\mu=0$ selections of
support $\le3$ have $C=1$, while $\{0,1,2\}$ at weights $2{:}3{:}2$ attains
the exact three-point optimum $C=\frac12$.
\end{proposition}

\begin{proposition}[Ratio-of-averages obstruction]\label{prop:averaging}
A second natural template also fails: retaining the original weights on the
selected support and comparing only the aggregate second moments
$\sum_S\Wnorm{G_S}^2$ and $\sum_SD_S^2$ over all $k$-subsets does not yield
the constant $\frac58$. For an edge $e=\{i,j\}$ put $w_e=p_ip_jd_{ij}$ and
$y_e=w_eB_{ij}$ if $t_i\ne t_j$, $y_e=0$ if $t_i=t_j$. Define
$U=\sum_ew_e^2$, $Q=\sum_e\Wnorm{y_e}^2$, $R=4\sum_ip_i^2$,
$V=2\sum_ip_i^2|z_i|^2$, $D_S=\sum_{e\subset S}w_e$,
$G_S=\sum_{e\subset S}y_e$. For $2\le k\le N$, define the binomials
$A=\binom{N-2}{k-2}$, $B=\binom{N-3}{k-3}$, $H=\binom{N-4}{k-4}$ (zero when
the lower index is negative). Then the following aggregate identities over all
$k$-subsets $S$ hold:
\[
  \sum_{|S|=k}\Wnorm{G_S}^2=(A-2B+H)Q+(B-H)V,\qquad
  \sum_{|S|=k}D_S^2=H+(A-2B+H)U+(B-H)R .
\]
For the fifth-roots system $t_j=\zeta_5^j$, $z_j=t_j^2$, $p_j=\frac15$ one
gets $U=\frac3{25},R=\frac45,Q=\frac15,V=\frac25$. The ratio of the two
aggregates is therefore $\frac{Q+V}{U+R}=\frac{15}{23}>\frac58$, so the aggregate
comparison cannot certify $\frac58$ on this system. This does not say that
every original-weight triangle is bad: on the same system the
original-weight support $\{0,1,3\}$ has $C/E=\frac{195-57\sqrt5}{242}
\approx0.279<\frac58$. The obstruction is to the aggregate averaging
argument, not to original-weight witnesses in general.
\end{proposition}

\begin{proof}[Proof of Propositions~\ref{prop:mu0} and \ref{prop:averaging}]
For the heptagon, $\mu=m_1,\lambda=m_2,\nu=m_3$ where
$m_r=\sum_jq_j\zeta^{rj}$, and Parseval gives
$1+2\sum_{r=1}^3|m_r|^2=7\sum_jq_j^2\ge7/k$. With $m_1=0$,
$C=|m_2|^2+|m_3|^2$. For the unconstrained witnesses: $q=\frac13$ on
$\{0,1,3\}$ gives, after reduction with $1+\zeta+\dots+\zeta^6=0$, the value
$C=\frac27$. For $q=\frac14$ on $\{0,1,2,3\}$, substituting
$m_r=\frac14\sum_{j=0}^{3}\zeta^{rj}$ into
$C=\bigl(|m_3-m_1m_2|^2+|m_2-\bar m_1m_3|^2\bigr)/(1-|m_1|^2)^2$ and
expressing everything through $x=2\cos\frac{2\pi}7$ (which satisfies
$x^3+x^2-2x-1=0$) shows that $\frac14-C$ has the sign of
$33x^2+71x-36>0$. For the hexagon, the $\mu=0$ supports of size $\le3$ are
the antipodal pairs and the two alternating triangles, all with $C=1$. The
three-point optimum follows from Lemma~\ref{lem:closure} by evaluating
$\dist^2(0,K_S)$ on the three dihedral types of triangles: the cyclic gap
types $(1,1,4)$ and $(1,2,3)$ give $\frac12$, while $(2,2,2)$ gives $1$. The value
$\frac12$ is attained on $\{0,1,2\}$ by the weights $2{:}3{:}2$. The
subset averages follow by counting occurrences of an edge, an adjacent edge
pair, and a disjoint edge pair in $k$-subsets, together with the cross-term
bookkeeping $\mathrm{Adj}=V-2Q$, $\mathrm{Disj}=Q-V$ (from
$\sum_ey_e=0$) and their scalar analogues. The fifth-roots values use
$\Wnorm{B_{ij}}^2=5-d_{ij}$ and the two edge lengths
$d_\pm=\frac{5\pm\sqrt5}2$.
\end{proof}

\subsection{Closed classes, a necessary condition, and records}\label{loc:proof-classes}

\begin{restate}{Proposition~\ref{prop:classes} (Classes on which Conjecture~\ref{conj:CE} holds; restated)}
The three-point inequality of Conjecture~\ref{conj:CE} holds for all moment
systems with at most four atoms, and the four-point inequality for all
systems with at most five atoms. Both inequalities hold for all two-direction
systems, and the four-point inequality holds for all antipodally paired
systems (disjoint unions of pairs $\{(t,z),(-t,z)\}$ with equal $z$-values
inside each pair).
\end{restate}

\begin{proof}
Systems with fewer atoms than the budget are trivial: $q=p$ gives $C=0$.
If $E=0$, then every $z_i=0$. In this zero-energy case, $\E_pt=0$
ensures that at least two distinct directions occur, and choosing one atom
from each of two such directions gives $|\mu|<1$ and $\nu=\lambda=0$, hence $C=0$.
We may therefore assume $E>0$.
If some atom has $p_i\ge\frac12$, then $\E_pt=0$ forces $p_i=\frac12$ and
all other atoms antipodal to $t_i$, so the system is two-directional and
Proposition~\ref{prop:twodir} applies. Otherwise every $p_i<\frac12$, and
deleting a single atom $i$ and renormalizing the original weights on the
others gives an admissible selection, since $|\mu|=p_i/(1-p_i)<1$.
Realize the system by a whitened dataset without zero-feature points and
with $\sum_i|\xi_i|^2=1$, so that $\sum_ix_ix_i^\top=\frac12I_2$
(Appendix~\ref{subsec:dictionary}). In the full least-squares problem, atom
$i$ then has leverage $h_i=x_i^\top(\sum_jx_jx_j^\top)^{-1}x_i=2p_i$, and
the leave-one-out identity (Sherman--Morrison) gives the exact deletion cost
\[
  C_{-i}=p_i|z_i|^2\,\frac{h_i}{(1-h_i)^2}.
\]
Write $\theta_i=p_i|z_i|^2/E$, so $\sum_i\theta_i=1$, and let $s$ be the
number of atoms. Since $\sum_ih_i=2$, Cauchy--Schwarz gives
$\sum_i\frac{(1-h_i)^2}{h_i}=\sum_i\frac1{h_i}-2s+2\ge\frac{s^2}2-2s+2
=\frac{(s-2)^2}2$, so not every $C_{-i}$ can exceed $\frac{2E}{(s-2)^2}$
(summing $\theta_i>\frac{2(1-h_i)^2}{(s-2)^2h_i}$ over $i$ would give
$1>1$). For $s=4$ this yields a three-atom selection with
$C\le\frac E2\le\frac{5E}8$, and for $s=5$ a four-atom selection with
$C\le\frac{2E}9\le\frac E4$. Two-direction systems are covered by
Proposition~\ref{prop:twodir}. For an antipodally paired system, select two
pairs and give the two atoms of each pair equal weight. Then $\mu=0$ and
$\lambda=0$, so $C=|\nu|^2$ with $\nu$ a convex combination of the two pair
values $z$. The pair values, weighted by the pair masses, form a zero-mean
system in $\C\cong\R^2$ with second moment $E$, and Lemma~\ref{lem:mmean}
with $M=2$, $k=2$ provides two pairs with $|\nu|^2\le E/4$.
\end{proof}

\paragraph{A necessary condition and exact records.}
By Lemma~\ref{lem:star}, any counterexample to the three-point inequality
must satisfy $5|z_i|^2-2L_i>3E$ for \emph{every} atom $i$ (for pairwise
distinct directions: $|z_i|^2>\frac{3E}5$ for all $i$). Exact records: the
hexagon three-point optimum is $\frac12$, while its four-point value is $0$
(the weights $(1,2,2,1)/6$ on four consecutive atoms give $\nu=\lambda=0$
and $|\mu|=1/\sqrt3$). The five-atom two-direction worst case is exactly
$\frac47$.

%% file: app-kpoint.tex
\section{The general \texorpdfstring{$k$}{k}-point mean lemma}
\label{sec:kpoint}

Lemma~\ref{lem:mmean} is the case $k=M$ of a sharp family of sparsification
bounds. We record this family both for its own interest and for its
applications: it supplies the sharp per-line constants for the allocation
problems mentioned in Section~\ref{sec:discussion} and settles the whole profile for
$d=1$ (Corollary~\ref{cor:d1}). The proof is elementary and self-contained.
\emph{None of Theorems~\ref{thm:threshold}, \ref{thm:lowbudget},
\ref{thm:nearthreshold} or \ref{thm:intervals} depends on this appendix}.
The only case used elsewhere is $k=M$ (Lemma~\ref{lem:mmean}). In the
boundary variance fact and in Sublemma~G below, $n$ denotes the dimension
of a simplex. In Corollary~\ref{cor:d1}, it again denotes the budget.

\paragraph{Relation to other sparsification bounds.} For $k=2$ and $M\ge3$,
the constant $\frac{M-1}{2M}$ of Theorem~\ref{thm:kpoint} coincides with,
and follows from, the bound in Theorem~2 of \citep{SharpPair} after the circuit
reduction used below. The cited theorem gives an equal-weight two-point
bound for zero-mean distributions on at most $s$ atoms, with constant
$\max\{\frac13,\frac{s-2}{2(s-1)}\}$. The reduction restricts the circuit
support to $s\le M+1$. For $M=2$, Lemma~\ref{lem:mmean}
gives $\frac14$ for arbitrary convex weights on two points. This feasible
set differs from the equal-weight pairs behind the constant $\frac13$
in the cited result, which also does not cover general $k$.
A radius-normalized, equal-weight analogue concerns the worst squared
distance from a point of a simplex in the unit ball to a $k$-fold
barycenter of its vertices. This is the sharp approximate Carath\'eodory
theorem of \citep{Tinarrage}. That statement normalizes by the radius of a
ball containing the simplex and uses equal weights, whereas
Theorem~\ref{thm:kpoint} normalizes by the variance $E$ and allows
arbitrary convex weights. Thus neither is a formal special case of the
other. Theorem~\ref{thm:kpoint} does imply a skeleton-distance statement
with radius normalization and a weaker, face-distance conclusion. If the vertices $v_i$ of an $M$-simplex lie in the
unit ball and $y=\sum_iq_iv_i$ with $q_i>0$, then
$E=\sum_iq_i\Wnorm{v_i-y}^2=\sum_iq_i\Wnorm{v_i}^2-\Wnorm y^2\le1-\Wnorm y^2$,
so some face $F$ spanned by at most $k$ vertices satisfies
$\dist(y,F)^2\le\frac{M+1-k}{kM}\,(1-\Wnorm y^2)$.

\begin{theorem}[$k$-point mean lemma]\label{thm:kpoint}
Let $M\ge1$ and let $z_i\in\R^M$ carry weights $p_i>0$, $\sum_ip_i=1$, with
$\sum_ip_iz_i=0$, and $E=\sum_ip_i\Wnorm{z_i}^2$. For every
$1\le k\le M+1$ there is a convex combination $\delta$ of at most $k$ of the
points with
\[
  \Wnorm{\delta}^2\ \le\ E\cdot\frac{M+1-k}{kM}.
\]
The constant is sharp: for the uniformly weighted regular $M$-simplex, every
convex combination supported on $k$ points has squared norm at least
$E\frac{M+1-k}{kM}$, with equality for $k$ vertices at equal weights.
\end{theorem}

The proof has two steps: (i) reduce to a single positive circuit of support
$s\le M+1$ (extreme points of the zero-mean polytope, as in
Lemma~\ref{lem:mmean}); (ii) apply the following sparsification lemma with
$a_i=z_i$.

\begin{lemma}[Simplex skeleton sparsification]\label{lem:sparsify}
Let $s\ge2$, let $a_1,\dots,a_s$ lie in a Hilbert space, let $q$ be a
probability vector with full support $s$, and let
$V_q=\sum_iq_i\Wnorm{a_i-\bar a_q}^2$ where $\bar a_q=\sum_iq_ia_i$. For
every $1\le k\le s$ there is a probability vector $\lambda$ with
$|\supp\lambda|\le k$ and
\[
  \Bigl\lVert\,\sum_i\lambda_ia_i-\bar a_q\Bigr\rVert^2
  \ \le\ \frac{s-k}{k(s-1)}\,V_q .
\]
\end{lemma}

\begin{proof}[Proof (the two nontrivial steps are stated in full)]
We construct a martingale of probability vectors that successively eliminates
coordinates. The key step is the following fact.

\smallskip
\noindent\textbf{Boundary variance fact.} \emph{Let $q$ have $r\ge2$
positive coordinates and let $H\succeq0$ be a quadratic form. Then there is a
random probability vector $Q_H$ supported on the boundary
$\partial\Delta_r=\{u\in\Delta_r:\min_iu_i=0\}$ with $\E[Q_H]=q$ and}
\[
  \E\bigl[(Q_H-q)^\top H(Q_H-q)\bigr]\ \le\
  \frac{\tr\bigl(H(\Diag q-qq^\top)\bigr)}{(r-1)^2}.
\]
(One boundary variable per form $H$ suffices for our purposes.)

\emph{Proof of the fact.} The left-hand side, minimized over admissible laws,
is a finite moment problem on the compact set $\partial\Delta_r$.
An optimal law with finite support exists. Strong duality gives
\[
  \Phi_H(q)=\sup\Bigl\{\ell(q):\ \ell\ \text{affine},\
  \ell(u)\le(u-q)^\top H(u-q)\ \forall u\in\partial\Delta_r\Bigr\},
\]
because the set of achievable pairs (mean, cost) is compact and convex and
$q$ is a relative interior point of $\Delta_r$, so the supporting hyperplane
of that set at the optimal value is non-vertical.
Fix a feasible $\ell$ and write $a=\ell(q)$. If $a\le0$, the desired bound
$(r-1)^2a\le\tr(H(\Diag q-qq^\top))$ holds because the right-hand side is
nonnegative. Otherwise, replace $H$ on the tangent space
$T=\{\ones^\top y=0\}$ by $H_\varepsilon=H+\varepsilon P_T$, which is
positive definite on $T$ and keeps $\ell$ feasible. In this positive-value case,
the bound for $H_\varepsilon$ implies the bound for $H$ as $\varepsilon\downarrow0$.
We may therefore assume $H\succ0$ on $T$. Writing $\ell(q+y)=a+2\langle c,y\rangle_H$
($c\in T$) and $R^2=a+\Wnorm{c}_H^2>0$, feasibility says that the open
$H$-ball $B_H(c,R)\subseteq T$ misses $\partial\Delta_r-q$. This ball
contains $0$ (as $a>0$), a relative interior point of the convex
set $\Delta_r-q$. Since the ball is connected, $B_H(c,R)\subseteq\Delta_r-q$.
The desired bound
$(r-1)^2\ell(q)\le\tr(H(\Diag q-qq^\top))
=\sum_iq_i(e_i-q)^\top H(e_i-q)$ then follows by applying the sublemma below
in the Euclidean structure $\langle\cdot,\cdot\rangle_H$ on $T$, with $n=r-1$, the simplex
$\Delta_r-q$, the ball $B_H(c,R)$, and the barycentric representation
$0=\sum_iq_i(e_i-q)$.

\smallskip
\noindent\textbf{Sublemma G (simplex--ball interpolation).}
\emph{Let $n\ge1$, let an $n$-simplex $T\subset\R^n$ contain the ball
$B(c,R)$, assume $\Wnorm{c}<R$, and let $0=\sum_iq_iv_i$ be the barycentric
representation of the origin in the vertices $v_i$ of $T$. Then}
\[
  \sum_iq_i\Wnorm{v_i}^2\ \ge\ n^2\,(R^2-\Wnorm c^2).
\]

\emph{Proof of Sublemma G.} Normalize the ball to the unit ball centered at
the origin by translating by $-c$ and dividing by $R$. The vertices become $V_i$,
and the original origin becomes $y=-c/R$ with $\beta=\Wnorm{y}<1$.
We have $y=\sum_iq_iV_i$ with the same barycentric weights, so the claim becomes
$\mathcal V_T(y):=\sum_iq_i\Wnorm{V_i-y}^2\ge n^2(1-\beta^2)$. Since $T$
contains the unit ball, the facet of $T$ opposite $V_i$ lies on a hyperplane
$\{x:\langle b_i,x\rangle=1\}$ with $\Wnorm{b_i}\le1$, and
$T=\{x:\langle b_i,x\rangle\le1\ \forall i\}$. Let $\lambda_i$ be the $i$-th
barycentric coordinate function of $T$ and $p_i=\lambda_i(0)>0$. As
$\lambda_i$ is affine and vanishes on the facet opposite $V_i$,
$\lambda_i(x)=p_i(1-\langle b_i,x\rangle)$. Comparing constant and linear
terms in the identities $\sum_i\lambda_i(x)=1$ and
$\sum_i\lambda_i(x)V_i=x$ gives
\[
  \sum_ip_i=1,\qquad \sum_ip_ib_i=0,\qquad \sum_ip_iV_i=0,\qquad
  \sum_ip_iV_ib_i^\top=-I_n .
\]
Put $t_i=\langle b_i,y\rangle$, so $|t_i|\le\beta$ and
$q_i=\lambda_i(y)=p_i(1-t_i)>0$. Taking the trace of the last identity and
using the second gives $\sum_ip_i\langle V_i-y,b_i\rangle=-n$. Then
Cauchy--Schwarz gives
\[
  n^2=\Bigl|\sum_ip_i\langle V_i-y,b_i\rangle\Bigr|^2
  \le\Bigl(\sum_iq_i\Wnorm{V_i-y}^2\Bigr)
      \Bigl(\sum_i\frac{p_i^2}{q_i}\Wnorm{b_i}^2\Bigr)
  \le\mathcal V_T(y)\sum_i\frac{p_i}{1-t_i}.
\]
Finally $\frac1{1-t_i}\le\frac{1+t_i}{1-\beta^2}$ (equivalent to
$t_i^2\le\beta^2$) and $\sum_ip_it_i=\langle\sum_ip_ib_i,y\rangle=0$ give
$\sum_i\frac{p_i}{1-t_i}\le\frac1{1-\beta^2}$. Hence
$\mathcal V_T(y)\ge n^2(1-\beta^2)$. This proves Sublemma G, hence the
boundary variance fact. \hfill$\square$

\smallskip
(For $n\ge2$, Sublemma G can also be derived from the Grace--Danielsson
inequality $(\mathcal R-nr)(\mathcal R+(n-2)r)\ge\omega^2$ between the
circumradius $\mathcal R$, the inradius $r$ and the distance $\omega$ of the
centers of a simplex, proved in all dimensions by Drozdov \citep{Drozdov}.
The argument above does not use this inequality.)

\smallskip
Now iterate: starting from $q$ with support $s$, apply the fact with
$H=\mathsf A^\top\mathsf A$ (where $\mathsf A=[a_1\cdots a_s]$), obtaining a
boundary law $Q$ with $\E[\bar a_Q]=\bar a_q$ and
$\E\Wnorm{\bar a_Q-\bar a_q}^2\le V_q/(r-1)^2$. The total-variance identity
$\E[V_Q]=V_q-\E\Wnorm{\bar a_Q-\bar a_q}^2$ gives
$\E[V_Q]\ge\bigl(1-\tfrac1{(r-1)^2}\bigr)V_q$. Stop as soon as the support
has size at most $k$. Set $f_r=1-\frac1{(r-1)^2}$ and
$\Pi_j=\prod_{t=k+1}^{j}f_t$ (with $\Pi_j=1$ for $j\le k$), which is
nonincreasing in $j$. By backward induction on the support size $r$, the
terminal variance obeys $\E[V_\Lambda]\ge\Pi_rV$ for every starting law of
support $r$ and variance $V$. A step may eliminate several coordinates at once,
but if the support drops from $r$ to $t\le r-1$, then
$\E[V_\Lambda\mid Q]\ge\Pi_tV_Q\ge\Pi_{r-1}V_Q$, whence
$\E[V_\Lambda]\ge\Pi_{r-1}\E[V_Q]\ge\Pi_{r-1}f_rV_q=\Pi_rV_q$. Since the
mean sequence is a martingale with orthogonal increments,
$\E\Wnorm{\bar a_\Lambda-\bar a_q}^2=V_q-\E[V_\Lambda]
\le\bigl(1-\Pi_s\bigr)V_q$, and the telescoping product
$\Pi_s=\prod_{r=k+1}^{s}\frac{r(r-2)}{(r-1)^2}=\frac{s(k-1)}{k(s-1)}$
yields the bound $\frac{s-k}{k(s-1)}V_q$. Some realization $\lambda$ of the
terminal law achieves it.
\end{proof}

{\hypertarget{penultimate-proof-g1}{}}%
\begin{proof}[Proof of Theorem~\ref{thm:kpoint}]
Reduce to a circuit $q$ of support $s\le M+1$ and second moment $E_q\le E$
exactly as in Lemma~\ref{lem:mmean}, handling the degenerate branches
$E_q=0$ and $s\le k$ as in that argument. In the remaining case,
$s>k\ge1$, so $s\ge2$. Apply
Lemma~\ref{lem:sparsify} with $a_i=z_i$,
$\bar a_q=0$: some $k$-sparse convex $\delta$ has
$\Wnorm\delta^2\le\frac{s-k}{k(s-1)}E_q$. Since $s\mapsto\frac{s-k}{s-1}$ is
nondecreasing for $k\ge1$, the constant is at most
$\frac{M+1-k}{kM}$. To prove sharpness, note that for regular simplex vertices,
$\Wnorm{\sum_i\alpha_iv_i}^2=R^2\bigl(\tfrac{M+1}M\sum_i\alpha_i^2-\tfrac1M\bigr)
\ge R^2\tfrac{M+1-k}{kM}$ whenever $|\supp\alpha|\le k$, using
$\sum\alpha_i^2\ge1/k$.
\end{proof}

\begin{corollary}[The full profile for $d=1$]\label{cor:d1}
For all $m\ge1$ and $n\ge1$,
\[
  \Fw(1,m,n)=1+\frac{m+1-n}{mn}\ \ (1\le n\le m+1),
  \qquad
  \Fw(1,m,n)=1\ \ (n\ge m+1).
\]
\end{corollary}

{\hypertarget{penultimate-proof-g3}{}}%
\begin{proof}
\emph{Upper bound.} Fix a dataset with $d=1$. If all $x_i=0$ the ratio is
$1$ (both rules return $W=0$). Otherwise, let $I=\{i:x_i\ne0\}$,
$g=\sum_{i\in I}x_i^2$, $p_i=x_i^2/g$, $a_i=y_i/x_i\in\R^m$ and
$\mu=\sum_{i\in I}p_ia_i$. The full-data minimizer is the column
$W^\circ=\mu$. With $E=\sum_{i\in I}p_i\Wnorm{a_i-\mu}^2$, we have
$gE\le N\Lstar$ (the difference is the loss on zero-feature points). If
$E=0$, a single point of $I$ recovers $W^\circ$. Otherwise, for $n\le m+1$,
Theorem~\ref{thm:kpoint} applied to $z_i=a_i-\mu$ gives a convex $\alpha$
supported on at most $n$ points of $I$ with
$\Wnorm{\sum_i\alpha_ia_i-\mu}^2\le E\frac{m+1-n}{mn}$. The selection
weights $c_i\propto\alpha_i/x_i^2$ have the unique minimizer
$\widehat W=\sum_i\alpha_ia_i$. By \eqref{eq:lossexp}, the excess loss is
$g\Wnorm{\widehat W-\mu}^2/N\le\Lstar\frac{m+1-n}{mn}$.
\emph{Lower bound.} Take $x_i\equiv1$ and labels the vertices of a
uniformly weighted regular $m$-simplex centered at $0$. Then $W^\circ=0$
and $\Lstar=E$. Every selection of at most $n$ points returns a convex
combination of at most $n$ vertices, of squared norm at least
$E\frac{m+1-n}{mn}$ by the sharpness part of Theorem~\ref{thm:kpoint}. For
$n\ge m+1$ the value $1$ is Theorem~\ref{thm:threshold}.
\end{proof}

At $n=1$ and $n=m$ the corollary reproduces $\Fw(1,m,1)=2=d+1$ and
$\Fw(1,m,m)=1+\frac1{m^2}=1+\frac1{dm^2}$, in agreement with
Theorems~\ref{thm:lowbudget} and \ref{thm:nearthreshold}.

%% file: expression-proof-locations.tex
\par{\pdfbookmark[1]{Results and proof locations}{penultimate-results-index}}%
\noindent\begin{minipage}{\linewidth}
\centering
\textbf{Results and proof locations}\par\smallskip
\begin{tabular}{@{}lll@{}}
\textbf{Object} & \textbf{Statement} & \textbf{Original proof location}\\
Theorem~\ref{thm:upper} & Section~\ref{sec:threshold} & Appendix~\ref{loc:proof-upper}\\
Theorem~\ref{thm:axial} & Section~\ref{sec:threshold} & Appendix~\ref{loc:proof-axial}\\
Theorem~\ref{thm:threshold} & Section~\ref{sec:threshold} & Appendix~\ref{loc:proof-axial}\\
Theorem~\ref{thm:rigidity} & Section~\ref{sec:rigidity} & Appendix~\ref{loc:proof-rigidity}\\
Theorem~\ref{thm:lines} & Section~\ref{sec:rigidity} & Appendix~\ref{loc:proof-lines}\\
Lemma~\ref{lem:mmean} & Section~\ref{sec:nearthreshold} & Appendix~\ref{loc:proof-mean}\\
Theorem~\ref{thm:nearthreshold} & Section~\ref{sec:nearthreshold} & Appendix~\ref{loc:proof-near}\\
Proposition~\ref{prop:2dfails} & Section~\ref{sec:nearthreshold} & Appendix~\ref{subsec:2dfails}\\
Theorem~\ref{thm:lowbudget} & Section~\ref{sec:lowbudget} & Appendix~\ref{app:lowbudget}\\
Lemma~\ref{lem:interp} & Section~\ref{sec:intermediate} & Appendix~\ref{loc:proof-interp}\\
Proposition~\ref{prop:lower22} & Section~\ref{sec:intermediate} & Appendix~\ref{loc:proof-lower22}\\
Theorem~\ref{thm:intervals} & Section~\ref{sec:intermediate} & Appendix~\ref{loc:proof-intervals}\\
Proposition~\ref{prop:seven} & Section~\ref{sec:intermediate} & Appendix~\ref{loc:proof-seven}\\
Proposition~\ref{prop:classes} & Section~\ref{sec:intermediate} & Appendix~\ref{loc:proof-classes}\\
Theorem~\ref{thm:kpoint} & Appendix~\ref{sec:kpoint} & Appendix~\hyperlink{penultimate-proof-g1}{\ref*{sec:kpoint}}, theorem proof\\
Corollary~\ref{cor:d1} & Appendix~\ref{sec:kpoint} & Appendix~\hyperlink{penultimate-proof-g3}{\ref*{sec:kpoint}}, following statement\\
\end{tabular}
\end{minipage}\par\medskip
\noindent Conjecture~\ref{conj:CE} is stated in Section~\ref{sec:intermediate}.
Details appear in Appendix~\ref{subsec:conjecture}.
\par